\documentclass[11pt]{article} 
\usepackage{graphicx, amsmath,amsthm,amssymb,url}

\usepackage{mathtools}
\usepackage[multidot]{grffile}
\usepackage[normalem]{ulem} 

\usepackage{fullpage,etoolbox}
\usepackage{color}
\usepackage[square]{natbib}
\usepackage{hyperref}       
\hypersetup{
    colorlinks=true,
    linkcolor=blue,
    citecolor = blue,
    urlcolor = blue
}

\usepackage[algo2e,ruled]{algorithm2e}
\usepackage{makecell,booktabs,ctable}

\usepackage{verbatim,amsfonts,amssymb}
\usepackage[utf8]{inputenc}
\usepackage[english]{babel}
\usepackage{graphicx}
\usepackage{xspace}
\usepackage{amsthm}
\usepackage{amsmath}

\usepackage{amsmath,graphicx,amsfonts}
\usepackage{chapterbib}    
\usepackage{subcaption}
\usepackage{algorithm,algorithmic}

\usepackage{hyperref}
\usepackage[capitalize,noabbrev]{cleveref}

\def\E{\mathbb{E}}
\def\<{\langle}
\def\>{\rangle}

\def\col{\text{col}}

\def\gap{\text{gap}}

\newcommand\norm[1]{\left\lVert#1\right\rVert}

\theoremstyle{plain}
\newtheorem{theorem}{Theorem}[section]
\newtheorem{proposition}[theorem]{Proposition}
\newtheorem{lemma}[theorem]{Lemma}
\newtheorem{corollary}[theorem]{Corollary}
\theoremstyle{definition}
\newtheorem{defn}[theorem]{Definition}
\newtheorem{assumption}{Assumption}
\theoremstyle{remark}
\newtheorem{remark}[theorem]{Remark}

\newtheorem*{convention*}{Convention}

\DeclareMathOperator*{\esssup}{ess\,sup}

\begin{document}



\newcommand{\diag}{\textup{diag}}
\newcommand{\td}{\textup{TD}}

\newcommand{\algoname}{\textsc{SAC}}
\newcommand{\algonamefull}{Scalable Actor Critic}

\newcommand{\khop}{{ \kappa}}
\newcommand{\rhok}{{ \rho^{\khop+1}}}
\newcommand{\fk}{{ f(\khop)}}
\newcommand{\nik}{{ N_i^{\khop}}}
\newcommand{\njk}{{ N_j^{\khop}}}
\newcommand{\nminusik}{{ N_{-i}^{\khop}}}
\newcommand{\nminusjk}{{ N_{-j}^{\khop}}}

\newcommand{\final}[1]{{\color{red}#1}}



\title{Learning to Stabilize Unknown LTI Systems on a Single Trajectory under Stochastic Noise}

\author{Ziyi Zhang\footnote{Department of Electrical and Computer Engineering, Carnegie Mellon University. Email: \url{ziyizhan@andrew.cmu.edu}}\qquad Yorie Nakahira\footnote{Department of Electrical and Computer Engineering, Carnegie Mellon University. Email: \url{ynakahir@andrew.cmu.edu}} \qquad Guannan Qu\footnote{Department of Electrical and Computer Engineering, Carnegie Mellon University. Email: \url{gqu@andrew.cmu.edu}}}
\date{}
\maketitle

\begin{abstract}
    We study the problem of learning to stabilize unknown noisy Linear Time-Invariant (LTI) systems on a single trajectory. It is well known in the literature that the learn-to-stabilize problem suffers from exponential blow-up in which the state norm blows up in the order of  $\Theta(2^n)$ where $n$ is the state space dimension. This blow-up is due to the open-loop instability when exploring the $n$-dimensional state space. To address this issue, we develop a novel algorithm that decouples the unstable subspace of the LTI system from the stable subspace, based on which the algorithm only explores and stabilizes the unstable subspace, the dimension of which can be much smaller than $n$. With a new singular-value-decomposition(SVD)-based analytical framework, we prove that the system is stabilized before the state norm reaches $2^{O(k \log n)}$, where $k$ is the dimension of the unstable subspace. Critically, this bound avoids exponential blow-up in state dimension in the order of $\Theta(2^n)$ as in the previous works, and
to the best of our knowledge, this is the first paper to avoid exponential blow-up in dimension for stabilizing LTI systems with noise. 
\end{abstract}


%


\section{Introduction}
\label{sec:intro}
Driven by the success of machine learning and the practical engineering need in control, there has been a lot of 
interests in learning-based control of unknown dynamical systems~\cite{Beard97,Li22,Bradtke94,Krauth19,Dean20}. However, the existing methods commonly rely on the strong assumption of having access to a known stabilizing controller. This motivates the learning-to-stabilize problem, i.e. learning to stabilize an unknown dynamical system, particularly on a single trajectory, which has long been a challenging problem both in theory and for applications such as control of automatic vehicles and unmanned aerial vehicles (UAV). 


 Although many classical adaptive control approaches can solve the learn-to-stabilize problem and achieve asymptotic stability guarantees~\cite{Astrom96,Sun01}, it is well known that the learn-to-stabilize problem suffers from an issue known as \emph{exponential blow-up} during transients. As an example, \citet{Abbasi-Yadkori11} and \citet{Chen07} presented a model-based approach for learning to stabilize an unknown LTI system $x_{t+1} = A x_t  + Bu_t$. It first excites the system in open loop to learn the dynamics matrices $(A,B)$ and then designs the stabilizer. However, the initial excitation phase needs to run the system in open loop for at least $n$ steps before learning $(A,B)$ where $n$ is the dimension of the state space, because it takes at least $n$ samples to fully explore the $n$ dimensional state space. As a result, the state norm blows up to the order of $2^{\tilde{O}(n)}$ as the system may be unstable in open loop. Such an exponential blow-up can be catastrophic and has been observed in multiple papers~\cite{Abbasi-Yadkori11,Chen07,Lale20,Perdomo21,Tsiamis2021}. Further, it has also been shown that all general-purpose control algorithms suffer a worst-case regret of $2^{\Omega(n)}$ \cite{Chen07}.

Despite the exponential blow-up lower bound in \cite{Chen07}, it is a worst-case bound and does not rule out better results for specific systems. This motivates the following question: \emph{is it possible to exploit instance-specific properties to learn to stabilize a noisy LTI system without suffering from the worst-case exponential blow-up in $n$?} This problem has two challenges. First, in order to avoid the exponential blow-up, one can only collect $o(n)$ samples, based on which we can only get partial information on the dynamics. With only partial information about the system dynamics, it is difficult to stabilize it. Second, the noise in each step of the system is amplified by the open loop unstable system, causing strong statistical dependencies between states, which explode exponentially in a single trajectory. 

To solve the first challenge, we use the framework proposed in \citet{LTI}, which gave an algorithm that stabilizes a \emph{deterministic} LTI system with only $\Tilde{O}(k)$ state samples along a trajectory, where $k < n$ is the number of unstable eigenvalues of $A$. Therefore, \citet{LTI} offered an algorithm with state norm upper bounded by $2^{\tilde{O}(k)}$, which avoids the exponential blow-up $2^{\Tilde{O}(n)}$~\cite{Chen07,Tsiamis2021}. However, \citet{LTI} does not solve the second challenge as it assumes \emph{noiseless and deterministic} system dynamics. In addition, \citet{LTI} assumes that the control matrix has the same dimension as the instability index $k$. In other words, the system is \emph{fully actuated} when restricted to the unstable subspace. This assumption is also unrealistic in applications, as the dimension of control input is problem-specific and may not be equal to $k$. Particularly, many real-world systems are under-actuated, meaning that the control dimension can be much less than $k$. 

To solve the second challenge and address the limitations in \citet{LTI}, we need to determine a new method to approximate the unstable part of the system dynamics under stochastic noise and stabilize it with under-actuated control inputs. This is nontrivial as, for example, while some previous works have designed methods to approximate system dynamics from a noisy and blowing-up trajectory\cite{near_optimal_LDS, Simchowitz18}, these methods do not study how to separate the unstable part of the dynamics from the stable part and how to stabilize the system. The goal of this paper is to overcome these technical challenges and \emph{to learn-to-stabilize an unknown LTI system without the exponential blow-up state norm in noisy and under-actuated settings.} 

\textbf{Contribution.} 
In this paper, we develop a novel model-based algorithm, LTS\textsubscript{0}-N, to stabilize an unknown LTI system. We design a new singular-value-decomposition(SVD)-based subspace estimation technique to estimate the ``unstable'' part of system dynamics under noise perturbations and stabilize it. Using this new technique, we develop an analytical framework with the Davis-Kahan Theorem to estimate the error of subspace estimation, based on which we show the approach stabilizes the unknown dynamical system with state norm bounded by $2^{O(k \log k + \log(n-k) + m - \log\gap)}$, where $m$ is the dimension of control input, and $\gap$ is a constant depending on the spectral properties of $A$. Note that this bound avoids the worst-case exponential blow-up in state dimension $\Theta(2^n)$ and outperforms the state-of-the-art for stabilizing unknown noisy systems \cite{Lale20,Chen07}. Further, despite the challenge caused by strong stochastic dependencies, the aforementioned bound achieves a similar guarantee as the norm bound in \citet{LTI} for noiseless systems. In addition, as an improvement to \citet{LTI}, we do not place any requirement on dimensions of system dynamics matrices and maintain the same complexity for under-actuated system dynamics. 

\textbf{Related Work.} Our work is mostly related to learn-to-control with known stabilizing controllers and learn-to-stabilize on a single trajectory. In addition, we will also briefly cover system identification.

\textit{Adaptive control.} Adaptive control enjoys a long history of study~\cite{Astrom96,Sun01,Chen21}. Most classical adaptive control methods focus on asymptotic stability and do not provide finite sample analysis, and therefore do not study the exponential blow-up issue explicitly. The more recent work on non-asymptotic sample complexity of adaptive control has recongnized the exponential blow-up issue when a stabilizing controller is not known a priori~\cite{Chen07,Faradonbeh17,Lee23,Tsiamis2021,Tu18}. Specifically, the most typical strategy to stabilize an unknown dynamic system is to use past trajectory to estimate the system dynamics and then design the controller~\cite{Berberich20,Persis20,Liu23}. Therefore, those works need to run in an open loop for at least $O(n)$ steps before stabilizing, resulting in an exponential blow-up in the order of the state space dimension. Compared with those works, we can stabilize the system with fewer samples by identifying and stabilizing only the unstable subspace, thus avoiding the exponential blow-up. 

\textit{Learn to control with known controller.} There is abundant literature on stabilizing LTI under stochastic noise \cite{Bouazza21,converse_lyapunov, Kusii18,Li22}. One line of research uses the model-free approach to learn the optimal controller \cite{Fazel19,Joao20,Li22, Wang22, Zhang20}. Those algorithms typically require a known stabilization controller as an initialization point for policy search. Another line of research utilizes the model-based approach, which learns the system dynamics before designing the controller and also requires a known stabilizing controller \cite{Cohen19, Mania19, Plevrakis20,Zheng20}. Compared with those works, we focus on learn-to-stabilize, and the controller we obtain can serve as the initialization to existing learning-to-control works that require a known stabilizing controller. 

\textit{Learning to stabilize on multiple trajectories.} There are also works that do not assume open-loop stability and learn the full system dynamics before designing a stabilizing controller while requiring $\widetilde{\Theta}(n)$ complexity~\cite{Dean20,Tu18,Zheng201}, which is larger than $\widetilde{O}(k)$ of our work. Recently, a model-free approach via the policy gradient method offers a novel perspective with the same complexity~\cite{Perdomo21}. Those works do not face the same exponential blow-up issue since they allow multiple trajectories, i.e., the state can be ``reset'' to $0$. Compared with their work, we focus on the more challenging setting of stabilizing on a single trajectory. 

\textit{Learning to stabilize on a single trajectory.} Learning to stabilize for a linear system in an infinite time horizon is a classic problem in control \cite{Lai86, Chen89, Lai91}. There have been algorithms incurring regret of $2^{O(n)}O(\sqrt{T})$ which relies on assumptions of observability and strictly stable transition matrices \cite{Abbasi-Yadkori11,Ibrahimi12}. Some studies have improved the regret to $2^{\tilde{O}(n)} + \tilde{O}(\text{poly}(n)\sqrt{T})$ \cite{Chen07,Lale20}. Recently, \citet{LTI} proposed an algorithm that requires $\tilde{O}(k)$ samples but has assumptions on the dimension of $B$ and does not incorporate noise in the system dynamics. In this work, we propose an algorithm that has the same state norm bound as \citet{LTI} in a noisy and potentially under-actuated LTI system. 

\textit{System identification.} Our work is related to system identification, which focuses on determining system parameters \cite{Oymak18, near_optimal_LDS, Simchowitz18, Xing22}. Our work is related in that our approach also partially determines the system parameters before constructing the stabilizing controller. Compared to those works, we not just conduct the identification but also close the loop by stabilizing the system. 

\section{Problem Formulation}
\label{sec:problem_formulation}
\textbf{Notations.} In this paper, we use the $L^2$-norm as the default norm $\norm{\cdot}$. We use $M^*$ to represent the conjugate transpose of $M$, $e_i$ to denote the unit vector with $1$ at the $i$-th entry and $0$ everywhere else, and $\rho(\cdot)$ to denote the spectral radius of a matrix. We provide an indexing of notations at \Cref{Appendix:index}. 
We consider an LTI system $x_{t+1} = A x_t + B u_t + \eta_t$ where $x_t, \eta_t \in \mathbb{R}^n$ and $u_t \in \mathbb{R}^m$ are the state, noise, and control input at time step $t$, respectively. The system dynamics determined by $A$ and $B$ are \emph{unknown} to the learner. We further assume $\E[\eta_t] = 0$, and there exists constant $C \in \mathbb{R}^+$ such that $\norm{\eta_{t}} < C$ for all $t \in \mathbb{N}$.\footnote{  The assumption on boundedness of noise can be loosened to sub-Gaussian random variables at the cost of a slightly more complicated proof. Indeed, in the simulation in \Cref{sec:simulation}, we show our algorithm stabilizes an LTI system with additive Gaussian noise.}

  

The goal of the learning is to stabilize the system with a learned controller, defined as follows:
\begin{defn}[Stabilizing controller]
\label{defn:stb_cont}
    Control rule $(u_t)$ is called a \textbf{stabilizing controller} if and only if the closed-loop system $x_{t+1} = A x_t + B u_t + \eta_t$ is ultimately bounded; i.e. when $\Vert \eta_t\Vert\leq C$ for all $t$, $\lim\sup_{t \rightarrow \infty} \norm{x_t} < C_n$ is guaranteed in the closed-loop system for some $C_n \in \mathbb{R}^+$. 
\end{defn}

The learner is allowed to learn the system by interacting with it on a single trajectory. More specifically, the learner can observe $x_t$ and freely determine $u_t$. In this paper, we make the standard assumption that $(A,B)$ is controllable. We also assume $x_0 = 0$ for simplicity of proof. Our proof can be easily generalized to nonzero initial conditions.  

\textbf{Exponential blow-up.} Although there are many existing works in the learn-to-stabilize problem, including classical adaptive control \cite{Sun01} or more recent learning-based control papers \cite{Abbasi-Yadkori11,Chen07,Ibrahimi12,Lale20}, it is widely recognized that any generic learn-to-stabilize algorithm inevitably causes exponential blow-up in the state norm as shown by the lower bound in \citet{Chen07} and \citet{Tsiamis2021}. This is because $\Theta(n)$ samples are mandatory to sufficiently explore the $n$-dimensional state space and estimate the system dynamics before designing a stabilizing controller is possible. In contrast to these existing approaches that estimate the full system, our approach breaks the lower-bound by isolating the smaller unstable subspace from the stable subspace, estimating the system dynamics in the unstable subspace under stochastic coupling, and showing that by stabilizing the "smaller" subspace, we can stabilize the entire state space. As such, our approach breaks the exponential blow-up lower-bound in the regime when the unstable subspace is has smaller dimension than $n$.

\section{Preliminaries}

Our approach uses the decomposition of the state space into stable and unstable subspace (introduced in \citet{LTI}), and we only conduct system identification and stabilization for the unstable subspace. In this section, we provide a review of these concepts. 
\subsection{Decomposition of the State Space}
\label{subsec:decompose}
Consider the open-loop system $x_{t+1} = A x_t$, where $A$ is diagonalizable. Let $\lambda_1,\cdots,\lambda_n$ denote the eigenvalues of $A$ such that \footnote{In practice, if $A$ does have the same eigenvalues, a slight perturbation will make $A$ have distinct eigenvalues, to which our method will apply. Further, a light perturbation will only introduce a $\log$ factor, as our dependence on the eigenvalue-related ``gap'' constant is only logarithmic, as shown in Theorem~\ref{thm:main}. }
$$|\lambda_1| > |\lambda_2| > \cdots > |\lambda_k| > 1 > |\lambda_{k+1}| > \cdots > |\lambda_n| .$$
We define the unstable subspace $E_u$ as the invariant subspace corresponding to the unstable eigenvalues $\lambda_1,\ldots,\lambda_k$ and the stable subspace $E_s$ as the invariant subspace corresponding to the stable eigenvalues $\lambda_{k+1},\ldots,\lambda_n$.


\textbf{The $E_u \oplus E_u^\perp$-decomposition.}
Let $P_1 \in \mathbb{R}^{n \times k}$ and $P_2 \in \mathbb{R}^{n \times (n-k)}$ denote the orthonormal bases of the unstable subspace $E_u$ and its orthogonal complement $E_u^\perp$, respectively, namely,
$$E_u = \text{col}(P_1), \quad E_u^\perp = \text{col}(P_2) .$$
Let $P = [P_1, P_2]$, which is also orthonormal and thus $P^{-1} = P^* = [P_1^*, P_2^*]^*$. Let $\Pi_1 := P_1 P_1^*$ and $\Pi_2 := P_2 P_2^*$ be the orthogonal projectors onto $E_u$ and $E_u^\perp$, respectively. With the above decomposition, we can transform the matrix $A$ into the two subspaces. Since $E_u$ is an invariant subspace with regard to $A$, there exists $M_1 \in \mathbb{R}^{k \times k}$, $\Delta \in \mathbb{R}^{k \times (n-k)}$, and $M_2 \in \mathbb{R}^{(n-k)\times(n-k)}$, such that 
\begin{equation*}
    AP = P 
    \begin{bmatrix}
        M_1 & \Delta \\ & M_2
    \end{bmatrix}
    \Leftrightarrow
    M:= \begin{bmatrix}
        M_1 & \Delta \\ & M_2
    \end{bmatrix}
    = 
    P^{-1} A P .
\end{equation*}
In the above decomposition, the top-left block $M_1 \in \mathbb{R}^{k \times k}$ acts on the unstable subspace, while $M_2$ acts on the stable subspace. Consequently, $M_1$ inherits all the unstable eigenvalues of $A$, and $M_2$ inherits all the stable eigenvalues. 

Finally, we examine the system dynamics after the above transformation. Let $y = [y_1^*, y_2^*]^*$ represent $x$ in the basis formed by the column vectors of $P$ after coordinate transformation (i.e. $x = Py$). The system dynamics after the transformation can be written as
\begin{equation}
\label{eqn:system_dynamics}
    \begin{bmatrix}
        y_{1,t+1} \\ y_{2,t+1}
    \end{bmatrix}
    =
    P^{-1} A P 
    \begin{bmatrix}
        y_{1,t} \\ y_{2,t}
    \end{bmatrix}
    +
    P^{-1}Bu_t
    + 
    \begin{bmatrix}
        P_1^*\\ P_2^*
    \end{bmatrix} \eta_t
    =
    \begin{bmatrix}
        M_1 & \Delta\\
        & M_2
    \end{bmatrix}
    \begin{bmatrix}
        y_{1,t}\\
        y_{2,t}
    \end{bmatrix}
    +
    \begin{bmatrix}
        P_1^* B \\ P_2^* B
    \end{bmatrix}
    u_t
    + \begin{bmatrix}
        P_1^*\\ P_2^*
    \end{bmatrix} \eta_t .
\end{equation}

\textbf{The $E_u \oplus E_s$-decomposition} As $M$ is not block diagonal, signified by the top-right $\Delta$ block, which represents how much a state shifts from $E_u^\perp$ to $E_u$ in one step, $E_u^\perp$ is in general \textit{not} an invariant subspace with respect to $A$ in the $E_u \oplus E_u^\perp$-decomposition. For convenience of analysis, we introduce another decomposition in the form of $E_u \oplus E_s$, where both $E_u$ and $E_s$ are invariant with respect to $A$. We also represent $E_u = \col(Q_1)$ and $E_s = \col(Q_2)$ by their \textit{orthonormal} bases, and define $Q := [Q_1 \quad Q_2]$. Since $E_u$ and $E_s$ are generally not orthogonal, we define $R := Q^{-1} = [R_1^*, R_2^*]^*$. The construction detail is further explained in Appendix A.1 of \citet{LTI}. 

\subsection{$\tau$-hop Control}

\label{section:tau-hop-control}

A $\tau$-hop controller only inputs non-zero control $u_t$ for once every $\tau$ steps, i.e. when $t = s\tau$, $s \in \mathbb{N}$. We inherit the $\tau$-hop mechanism introduced in \citet{LTI} but change the stopping time mechanism.
Let $\Tilde{x}_s := x_{s\tau}$ and $\Tilde{u}_s := u_{s\tau}$ denote state and control action $\tau$ time steps apart. We can then write the dynamics of the $\tau$-hop control system as: 
\begin{equation}
    \Tilde{x}_{s+1} = A^\tau \Tilde{x}_s + A^{\tau-1}B \Tilde{u}_s + \sum_{i = 0}^{\tau - 1} A^{i} \eta_{s\tau + i} .
\end{equation}
Let $\Tilde{y}_s$ denote the state under $E_u \oplus E_u^\perp$-decomposition, i.e. $\Tilde{y}_s = P^* \Tilde{x}_s$. The state evolution becomes
\begin{equation}
\label{eqn:system_y_tilde}
    \begin{split}
        \begin{bmatrix}
            \Tilde{y}_{1,s+1} \\ \Tilde{y}_{2,s+1}
        \end{bmatrix}
        =&
        P^{-1} A^\tau P 
        \begin{bmatrix}
            \Tilde{y}_{1,s} \\ \Tilde{y}_{2,s}
        \end{bmatrix}
        +
        P^{-1}A^{\tau-1} B\Tilde{u}_s 
        + \sum_{i=0}^{\tau-1} P^{-1} A^i \eta_{s\tau + i}
        \\
        =&
        M^{\tau}
        \begin{bmatrix}
            \Tilde{y}_{1,s}\\
            \Tilde{y}_{2,s}
        \end{bmatrix}
        +
        \begin{bmatrix}
            P_1^* A^{\tau-1} B \\ P_2^* A^{\tau-1} B
        \end{bmatrix}
        \Tilde{u}_{s}
        + \sum_{i = 0}^{\tau - 1} 
        \begin{bmatrix}
            P_1^* A^{i} \\ P_2^* A^{i}
        \end{bmatrix}
        \eta_{s\tau + i} .
    \end{split}
\end{equation}
We shall denote $B_\tau := P_1^* A^{\tau-1}B$ for simplicity, and
\begin{equation*}
    M^\tau = 
    \left(\begin{bmatrix}
        M_1 & \\ & M_2
    \end{bmatrix}
    +
    \begin{bmatrix}
        0 & \Delta \\ & 0
    \end{bmatrix}\right)^\tau 
    = 
    \begin{bmatrix}
        M_1^\tau & \sum_{i=1}^{\tau-1} M_1^i \Delta M_2^{\tau-1-i} \\
        & M_2^\tau
    \end{bmatrix}
    :=
    \begin{bmatrix}
        M_1^\tau & \Delta_\tau \\ & M_2^\tau
    \end{bmatrix}
    .
\end{equation*}
Now we use a state feedback controller $\Tilde{u}_s = K_1 \Tilde{y}_{1,s}$ in the $\tau$-hop control system to stabilize the system by acting on the unstable component $\Tilde{y}_{1,s}$. The closed-loop dynamics can be written as 

\begin{equation}
\label{eqn:tau_hop_closed_simplified}
    \Tilde{y}_{s+1} = 
    \begin{bmatrix}
        M_1^\tau + P_1^* A^{\tau-1} B K_1 & \Delta_\tau 
        \\
        P_2^* A^{\tau-1} B K_1 & M_2^\tau
    \end{bmatrix}
    \Tilde{y}_s + \sum_{i = 0}^{\tau - 1} P^{-1} A^{i} P \eta_{s\tau + i} .
\end{equation}

\section{Main Results}
\subsection{Algorithm}
In this section, we propose Learning to Stabilize from Zero with Noise (LTS\textsubscript{0}-N). 
The algorithm is divided into 4 stages: (i) learn an orthonormal basis $P_1$ of the unstable subspace $E_u$ (Stage 1); (ii) learn $M_1$, the restriction of $A$ onto the subspace $E_u$ (Stage 2); (iii) learn $B_\tau = P_1^* A^{\tau-1}B$ (Stage 3); and (iv) design a controller that seeks to stabilize the ``unstable'' $E_u$ subspace (Stage 4). This is formally described in \Cref{alg:LTS0}. We provide detailed descriptions of the four stages in LTS\textsubscript{0}-N. 

\begin{algorithm}[tb]
  \caption{LTS\textsubscript{0}-N: learning a $\tau$-hop stablilzing controller}
  \label{alg:LTS0}
\begin{algorithmic}[1]
  \STATE \textbf{Stage 1: learning the unstable subspace of $A$.}
  \STATE Run the system in open loop for $T$  steps and let $D \leftarrow [x_{1}, \cdots, x_{T}].$
  \STATE Compute the singular value decomposition of $D = U \Sigma V^*$. Let $\hat{P}_1 \leftarrow U^{(k)}$ be the top $k$ columns of $U$.
  \STATE Calculate $\hat{\Pi}_1 \leftarrow \hat{P}_1 \hat{P}_1^*$.
  \STATE \textbf{Stage 2: approximate $M_1$ on the unstable subspace.}
  \STATE Solve the least square problem $\hat{M}_1 \leftarrow \arg\min_{M_1 \in \mathbb{R}^{k \times k}} \mathcal{L}(M_1) := \sum_{t=0}^T \norm{\hat{P}_1^* x_{t+1} - M_1 \hat{P}_1^* x_t}^2$.
  \STATE \textbf{Stage 3: restore $B_\tau$ for $\tau$-hop control.}
  \FOR{$i = 1,\cdots,m$}
    \STATE \label{alg:stopping_time} Let the system run in open loops for $\omega_i$ steps until $\frac{\norm{(I - \hat{\Pi}_1)x_{t_i}}}{\norm{x_{t_i}}} < (1-\epsilon)\gamma$ and $\frac{C}{\norm{x_{t_i}}} < \delta$. 
    \STATE Run for $\tau$ more steps with initial $u_{t_i} = \alpha \norm{x_{t_i}}e_i$, where $t_i = T + \sum_{j=1}^i\omega_j + (i-1)\tau$.
 \ENDFOR
 \STATE Let $\hat{B}_\tau \leftarrow [\hat{b}_1,\cdots,\hat{b}_m]$, where the $i$-th column $\hat{b}_i \leftarrow \frac{1}{\alpha \norm{x_{t_i}}}\left(\hat{P}_1^* x_{t_i + \tau} - \hat{M}_1^\tau \hat{P}_1^* x_{t_i}\right).$
 \STATE \textbf{Stage 4: construct a $\tau$-hop stabilizing controller $K$.}
 \STATE Construct the $\tau$-hop stabilizing controller $\hat{K}_1$ from $\hat{M}_1^\tau$ and $\hat{B}_\tau$. 
\end{algorithmic}
\end{algorithm}

\textbf{Stage 1: Learning the unstable subspace of $A$.}
We let the system run in open-loop (with control input $u_t \equiv 0$) for $T$ time steps. Per the stable/unstable decomposition, the ratio between the norms of the state components in the unstable and stable subspace increases exponentially, and, very quickly, the state will lie ``almost'' in $E_u$. Consequently, the subspace spanned by the $T$ states, i.e. the column space of $D := [x_{1}, \cdots, x_{T}]$, is very close to $E_u$. 
Thus, we use the top $k$ left singular vectors of $D$ (the top $k$ eigenvectors of $DD^*$), denoted as $U^{(k)}$, as an estimate of the basis of the unstable subspace $\hat{P}_1$. In other words, we set $\hat{P}_1 = U^{(k)}$ and use it to construct the orthogonal projector onto $E_u$, namely $\hat{\Pi}_1 = U^{(k)}(U^{(k)})^*$, as an estimation of the projector $\Pi_1 = P_1 P_1^*$ onto $E_u$. 

\textbf{Stage 2: Learn $M_1$ on the unstable subspace.} Recall that $M_1$ is the system dynamics matrix for the subspace $E_u$ under $E_u \oplus E_u^\perp$-decomposition. Therefore, to estimate $M_1$, we first compute the projection of states $x_{1:T}$ on subspace $E_u$, i.e. $\hat{y}_{1,t} = \hat{P}_1^* x_{1,t}$ for $t = 1,\cdots,T$. Then we use least squares to estimate $M_1$, i.e. find $\hat{M}_1$ that minimizes the square loss:
\begin{equation}
\label{eqn:est_M_1}
    \begin{split}
        \mathcal{L}(\hat{M}_1) &:= \sum_{t = 0}^T \norm{\hat{y}_{1,t+1}- \hat{M}_1 \hat{y}_{1,t}}^2 .
    \end{split}
\end{equation}

\textbf{Stage 3: Learn $B_\tau$ for $\tau$-hop control.} In this stage, we estimate $B_\tau$, which quantifies the effect of control input on states in the unstable subspace $E_u$ (as discussed in \Cref{section:tau-hop-control}). Note that \eqref{eqn:system_y_tilde} shows
\begin{equation}
\label{eqn:y_for_b}
    y_{1,t_i + \tau} = M^\tau y_{1,t_i} + \Delta_\tau y_{2,t_i} + B_\tau u_{t_i} + \sum_{j = 1}^{\tau-1} M^{\tau-j} \eta_{1,t_i + j} + \Delta_{\tau-j} \eta_{2,t_i + j} .
\end{equation}
We estimate the columns of $B_\tau$ one by one. Specifically, we use a scaled unit vector $e_i$ as control input at time $t_i$, run the system in open loop for $\tau$ steps, and use \eqref{eqn:y_for_b} but simply ignore the $\Delta_{\tau}$ related terms to estimate $b_i$, the $i$-th column of $B_\tau$, as
\begin{equation}
\label{eqn:b}
    \hat{b}_i = \frac{1}{\norm{u_{t_i}}}\left(\hat{P}_1^* x_{t_i+\tau} - \hat{M}_1^\tau \hat{P}_1^* x_{t_i} \right) ,
\end{equation}
where $u_{t_i}$ is parallel to $e_i$ with magnitude $\alpha \norm{x_{t_i}}$ for normalization. Here, $\alpha$ is an adjustable constant to guarantee that the $E_s$-component does not increase too much to blur our estimation after injecting $u_{t_i}$. Since we ignored the $\Delta_\tau$ related terms in the estimation of $b_i$, to ensure that those terms do not cause much error in our estimation of $B_{\tau}$,  we let the system run in open loop for $\omega_i$ time steps before the estimation of $b_i$ starts. Here, $\omega_i$ is a stopping time (cf. Line \ref{alg:stopping_time} in \cref{alg:LTS0}). The purpose of the stepping time is to reduce the estimation error caused by the $\Delta_\tau$. For more details, see \Cref{prop:G6} in the proof.

\textbf{Stage 4: Construct a $\tau$-hop stabilizing controller $K$.} With the estimated $M_1^\tau$ and $B_\tau$ from the last stage, denoted as $\hat{M}_1^\tau$ and $\hat{B}_\tau$, the learner can choose any stabilization algorithm to find $\hat{K}_1$ by stabilizing the linear system 
\begin{equation*}
    \hat{\Tilde{y}}_{i+1} = \hat{M}_1^\tau \hat{\Tilde{y}}_i + \hat{B}_\tau \Tilde{u}_{i}, \qquad \Tilde{u}_i = \hat{K}_1 \hat{\Tilde{y}}_i ,
\end{equation*}
where the tilde in $\hat{\Tilde{y}}$ emphasizes the use of $\tau$-hop control and the hat emphasizes the use of estimated projector $\hat{P}_1$, which introduces an extra estimation error to the final closed-loop dynamics. As $\hat{K}_1$ is chosen by the learner, we denote $\mathcal{K}$ to be a constant such that $\norm{\hat{K}_1} < \mathcal{K}$. Furthermore, by \Cref{prop:controllable_Mtau}, there exists a positive definite matrix $\Bar{U}$ such that $\norm{\hat{M}_1^{\tau} - \hat{B}_{\tau} \hat{K}_1}_{\Bar{U}} := \mathcal{U} < 1$, where $\norm{\cdot}_{\Bar{U}}$ denotes the weighted norm induced by $\Bar{U}$. These user-defined constants are used in the proof of \Cref{thm:main}.

To sum up, \Cref{alg:LTS0} terminates in $T + \sum_{i=1}^m(1+\omega_i + \tau)$ time steps, where $\omega_i$ is the stopping time for the system to satisfy $\frac{\norm{(I - \hat{\Pi}_1)x_{t_i}}}{\norm{x_{t_i}}} < (1-\epsilon)\gamma$ and $\frac{C}{\norm{x_{t_i}}} < \delta$. 

\begin{remark}
    Our algorithm is different from the algorithm proposed in \citet{LTI} in three aspects. Firstly, to account for the noise, we do not directly use the span of consecutive $k$ vectors as the estimator for the unstable subspace. Instead, to identify the unstable subspace under noise, we utilize the singular value decomposition to identify the dominating state space in the trajectory and use that space as an estimation of $P_1$. Such an estimator requires a much more delicate analysis framework to bound the error based on Davis-Kahan Theorem, which we elaborate in \Cref{Appendix:proj_proof}. Secondly, the above algorithm generalizes the problem to an under-actuated setting, where the control matrix $B \in \mathbb{R}^{n \times m}$ with $m \neq k$. To achieve this, unlike \cite{LTI} we no longer try to cancel out the unstable matrix $M_1$, but rather allow the learner to choose the stabilization controller. We show in \Cref{sec:simulation} that our algorithm outperforms \citet{LTI} in an under-actuated setting in simulation. Thirdly, we use a stopping time to monitor the state norm in estimating $B_{\tau}$, so that our algorithm always terminates at the earliest possible time.
\end{remark}

\subsection{Stability Guarantee}
In this section, we formally state the assumptions and show our approach finds a stabilizing controller without suffering from exponential blow-up in $n$. Our first assumption is regarding the spectral properties of $A$, which requires distinct eigenvalues with specified eigengap.

\begin{assumption}[Spectral Property]
\label{assumption:eigengap}
    $A$ is diagonalizable with distinct eigenvalues $\lambda_1,\ldots,\lambda_n$ satisfying $|\lambda_1| > |\lambda_2| > \dots > |\lambda_k| > 1 > |\lambda_{k+1}| > \dots > |\lambda_n|$. 
\end{assumption}

We assume the learner knows the value of $k$. However, we point out that our algorithm works as long as the learner picks a value $\hat{k}$ at least as large as $k$. In order to provide guarantee to the estimation of the open-loop unstable system dynamics, we also need an assumption on the distribution of noise $\eta$.

\begin{assumption}[pdf of $\eta$]
    \label{assumption:pdf}
    Let $M_1 := \Bar{P}^{-1}J \Bar{P}$ denote the Jordan normal form of $M_1$, and $\Bar{P} := [\Bar{P}_1,\Bar{P}_2,\cdots,\Bar{P}_k]^*$.
    There exists $C_z \in \mathbb{R}$, such that the supremum of the probability distribution function (pdf) of $\left|\Bar{P}_i^* \sum_{j = 1}^t M_1^{-j} P_1^* \eta_j\right|$ is upper bounded almost everywhere,.i.e.
       $ \esssup \text{pdf} \left(\left|\Bar{P}_i^* \sum_{j = 1}^t M_1^{-j} P_1^* \eta_j\right|\right) < C_z$, 
    for all $i \in \{1,\dots,k\}$ and $t \in \mathbb{N}$. 
\end{assumption}
\Cref{assumption:pdf} holds for most common noise distributions, including bounded uniform distribution and Gaussian distributions(\Cref{lemm:upper_bd_Cz}). We further discuss this assumption in \Cref{Appendix:D1} and \ref{Appendix:Aux_D1}. 

With the above assumptions, our main result is as follows. 
\begin{theorem}
\label{thm:main}
Given a noisy LTI system $x_{t+1} = Ax_t + Bu_t + \eta_t$ subject to \Cref{assumption:eigengap}, \Cref{assumption:pdf}, and additionally, $|\lambda_1| |\lambda_{k+1}| < 1$. Further, denote $\gap := \left|\prod_{\substack{m_1 \neq m_2,\\m_1, m_2 \in \{1,\dots,k\} }}(\lambda_{m_1}^{-1} - \lambda_{m_2}^{-1})\right|.$ By running \cref{alg:LTS0} with parameters
$\gamma = O(1), \quad \delta = O(m^{-\frac{1}{2}}), \quad\tau = O(1), \quad \alpha = O(1) $, and $
        T = O\left(k \log k + \log(n-k) + \log m - \log\gap\right)$, the controller returned by \Cref{alg:LTS0} is a stabilizing controller. Further, Algorithm~\ref{alg:LTS0} guarantees that 
\begin{equation*}
    \norm{x_t} < \exp \left(O \left(k \log k + \log (n-k) + m - \log\gap\right)\right),
\end{equation*}
before termination. Here the big-O notation only shows dependence on $k,m$ and $n$, while omitting dependence on $C, C_z, |\lambda_1|, |\lambda_k|,|\lambda_{k+1}|, \theta$, $\mathcal{K}$, and $\mathcal{U}$. 
\end{theorem}

The precise bound given for each constant can be found at \eqref{eqn:tau_final},\eqref{eqn:bdd_gamma},\eqref{eqn:bdd_alpha}, and \eqref{eqn:bdd_delta} in the Appendix, and the bound for $T$ is given in \Cref{thm:projection}. Despite the more challenging setting with noises and potentially underactuated systems, \Cref{thm:main} achieves a similar guarantee as \citet{LTI}. Specifically, in the regime of $m=O(k)$,\footnote{We note that the regime of $m=O(k)$ is the most interesting regime as it covers the under-actuated setting, which is known to be more challenging. } the above Theorem shows that LTS\textsubscript{0}-N finds a stabilizing controller with an upper bound on state norm at $2^{\Tilde{O}(k)}$, which is better than the state-of-the-art $2^{\Theta(n)}$ complexity in the noisy settings. Therefore, our approach leverages instance specific properties (the dimension of unstable subspace $k$) to \emph{break the exponential lower bound \cite{Chen07} and learns to stabilize without the exponential blow-up in $n$ in noisy and under-actuated settings}. 

We also point out that constant $\gap$ is also $k$-dependent. In the worst case, the $\gap$ has an order of $2^{O(k^2)}$. This is still independent of $n$. We note that \citet{LTI} did not show explicit dependence on this constant. We leave it as future work whether this additional constant is essential or is an artifact of the proof. Moreover, our assumption that $|\lambda_1||\lambda_{k+1}| < 1$ is weaker than the assumption in \citet{LTI}, which requires $|\lambda_1|^2 |\lambda_{k+1}| < |\lambda_k|$. 

We demonstrate the effectiveness of our algorithm in simulation in \Cref{sec:simulation}, showing our algorithm's state norm does not blow-up with $n$ and also outperforms other benchmarks.

\section{Proof Outline}
In this section, we will give a high-level overview of the key proof ideas for the main theorem. The full proof details can be found in Appendix \ref{Appendix:Main}. 

\textbf{Proof Structure.} The proof is largely divided into four steps. In Step 1, we examine how accurately the learner estimates the unstable subspace $E_u$ in Stage 1. We will show that $\Pi_1, P_1$ can be estimated up to an error of $\epsilon$, $\delta$ respectively within $T = O \left(k \log k + \log(n-k) - \log \epsilon - \log\gap\right)$ steps, where $\delta:=\sqrt{2k}\epsilon$. In Step 2, we examine how accurately the learner estimates $M_1$. We show that $M_1$ can be estimated up to an error of $3\norm{A}\delta$. In Step 3, we examine the estimation error of $B_\tau$ in Stage 3. Lastly, in Step 4, we eventually show that the $\tau$-hop controller output by Algorithm \ref{alg:LTS0} makes the system stable. 

\textbf{Overview of Step 1.}
To upper bound the estimation errors in Stage 1, we use SVD to isolate the unstable subspace and use the Davis-Kahan Theorem to decouple the system dynamics from the noise perturbation. The bounds on $\norm{\Pi_1 - \hat{\Pi}_1}$ is shown in Theorem \ref{thm:projection}. 

\begin{theorem}
\label{thm:projection}
For a linear dynamic system with noise $x_{t+1} = A x_t + \eta_t$ satisfying \Cref{assumption:eigengap} and \Cref{assumption:pdf}, let $E_u$ be the unstable subspace of $A$, $k=\dim E_u$ be the instability index of the system and $\Pi_1$ be the orthogonal projector onto subspace $E_u$. Then for any $\epsilon > 0$, by running Stage 1 of Algorithm 1 with an arbitrary initial state for $T$ time steps, where 
\begin{equation*}
    T = O \left(k \log k + \log(n-k) - \log \epsilon - \log\gap\right),
\end{equation*}
 we get an estimation $\hat{\Pi}_1 = U^{(k)}(U^{(k)})^*$ with error $\norm{\hat{\Pi}_1 - \Pi_1} < \epsilon$. Here, the big-O notation only shows dependence on $k,n$ and $\epsilon$, while omitting dependence on $C, C_z, |\lambda_1|, |\lambda_k|,|\lambda_{k+1}|$, and $\theta$. 
\end{theorem}
The proof of \Cref{thm:projection} is deferred to \Cref{Appendix:proj_proof}.

\textbf{Overview of Step 2.} To upper bound the error in Stage 2, We upper bound the error in $\arg\min_{M_1} \sum_{t=0}^T \norm{(U^{(k)})^* x_{t+1} - M_1 (U^{(k)})^* x_t}^2$ and obtain the following proposition. 
\begin{proposition}
\label{prop:G2}
    Under the premise of Theorem~\ref{thm:main}, we have
    \begin{equation*}
        \norm{\hat{M}_1^\tau - M_1^\tau} \leq 3 \tau \norm{A} \zeta_{\epsilon_1} (A)^2(|\lambda_1|+\epsilon_1)^{\tau-1} \delta,
    \end{equation*}
    where $\zeta_{\epsilon_1}(A)$ is constant for Gelfand's formula defined in \Cref{lemma:Gelfand}, and we recall $\delta$ is the estimation error for $P_1$.
\end{proposition}

The proof in this step and the related lemmas and propositions are deferred to Appendix~\ref{Appendix:ls}.

\textbf{Overview of Step 3.} To bound the error in Stage 3, we upper bound the error in each column of $B_{\tau}$. In particular, we show that \eqref{eqn:b} generates an estimation of $B_{\tau}$ with an error in the same order as $\delta$. The detail is left to \Cref{prop:G6} in Appendix~\ref{Appendix:boundingB}. 

\textbf{Overview of Step 4.} To analyze the stability of the closed-loop system, we shall first write out the closed-loop dynamics under the $\tau$-hop controller. Recall in Section \ref{section:tau-hop-control}, we have defined $\Tilde{u}_s, \Tilde{x}_s, \Tilde{y}_s$ to be the control input, state in $x$-coordinates, and state in $y$-coordinates in the $\tau$-hop control system, respectively. Using those notations, the learned controller is obtained from the estimation of $M_1^\tau$ and $B_\tau$ by the learner with any stabilization algorithm (e.g. LQR, pole-placement). 

Therefore, the closed-loop, the closed-loop $\tau$-hop dynamics should be 
\begin{equation}
\label{eqn:tau_hop_closed}
    \Tilde{y}_{s+1} = \hat{L}
    \begin{bmatrix}
        \Tilde{y}_{1,s}\\ \Tilde{y}_{2,s}
    \end{bmatrix}
    + \sum_{i = 0}^{\tau - 1} P^{-1} A^{i} P \eta_{s\tau + i}
    :=
    \hat{L} \Tilde{y}_s + \sum_{i = 0}^{\tau - 1} 
    \begin{bmatrix}
        P_1^* A^{i} \\ P_2^* A^{i}
    \end{bmatrix}
    \eta_{s\tau + i} ,
\end{equation}
where 
\small
\begin{equation}
\label{eqn:L_hat}
    \hat{L} := \begin{bmatrix}
        M_1^\tau + P_1^* A^{\tau-1} B \hat{K}_1 \hat{P}_1^* P_1 &
        \Delta_\tau + P_1^* A^{\tau-1}B \hat{K}_1 \hat{P}_1^* P_2 \\
        P_2^* A^{\tau - 1} B \hat{K}_1 \hat{P}_1^* P_1 &
        M_2^\tau + P_2^* A^{\tau - 1} B \hat{K}_1 \hat{P}_1^* P_2
        \end{bmatrix}
        := \begin{bmatrix}
            \hat{L}_{1,1} & \hat{L}_{1,2} \\
            \hat{L}_{2,1} & \hat{L}_{2,2}
        \end{bmatrix} .
\end{equation}
\normalsize

We will show the above system to be ultimately bounded (i.e. $\rho(\hat{L}_\tau) < 1$). Note that $\hat{L}_\tau$ is given by a 2-by-2 block form, and we can utilize the following lemma for the spectral analysis of block matrices. 

\begin{lemma}
    For block matrices $A = \begin{bmatrix}
        A_1 & 0 \\ 0 & A_2
    \end{bmatrix}$, $E = \begin{bmatrix}
        0 & E_{12} \\ E_{21} & 0
    \end{bmatrix}$, 
    the spectral radii of $A$ and $A+E$ differ by at most $|\rho(A+E) - \rho(A)| \leq \chi(A+E) \norm{E_{12}}\norm{E_{21}}$, where $\chi(A+E)$ is a constant. 
\end{lemma}

The proof of the lemma can be found in existing literature such as \cite{Nakatsukasa18}. Therefore, we need to ensure the stability of the diagonal blocks of $\hat{L}$ and upper-bound the norms of the off-diagonal blocks via estimation of factors appearing in these blocks.
Complete proofs can be found in Appendices \ref{Appendix:Main}.

\section{Numerical simulation}
\label{sec:simulation}
Lastly, we include numerical simulations to demonstrate the performance of our algorithm. We consider an LTI system with additive noise
\begin{equation*}
    x_{t+1} = A x_{t} + B u_{t} + \eta_t, \quad \text{where } \eta_t \sim \mathcal{N}(0, \sigma^2 I), 
\end{equation*}
where $\sigma^2$ is the variance of the additive Gaussian noise at each step. Note we use unbounded Gaussian noise here, and noise with bounded uniform distribution would generate similar results. The dynamics matrix $B$ is generated randomly. Matrix $A$ is generated by $A = V \Lambda V^{-1}$, where $V$ is a randomly generated matrix, and $\Lambda$ is a diagonal matrix of eigenvalues generated uniformly at random from the interval that satisfies $|\lambda_1||\lambda_{k+1}| < 1$. 

In our first experiment, we compare the performance of LTS\textsubscript{0}-N in different settings (with different $n,\sigma$).  In each setting, we conduct 200 trials and record the minimal time steps it takes to stabilize the system, and the results are in \Cref{fig:simulation_a}. In our second experiment, we compare our proposed algorithm to three different algorithms: a classical self-tunning regulator in \citet{Astrom96}, black-box control proposed in \citet{Chen07}, and the LTS\textsubscript{0} algorithm proposed in \citet{LTI} and the results are in \Cref{fig:simulation_b}. 

\begin{figure*}[t!]
    \centering
    \begin{subfigure}[t]{0.40\textwidth}
        \centering
        \includegraphics[width=\textwidth]{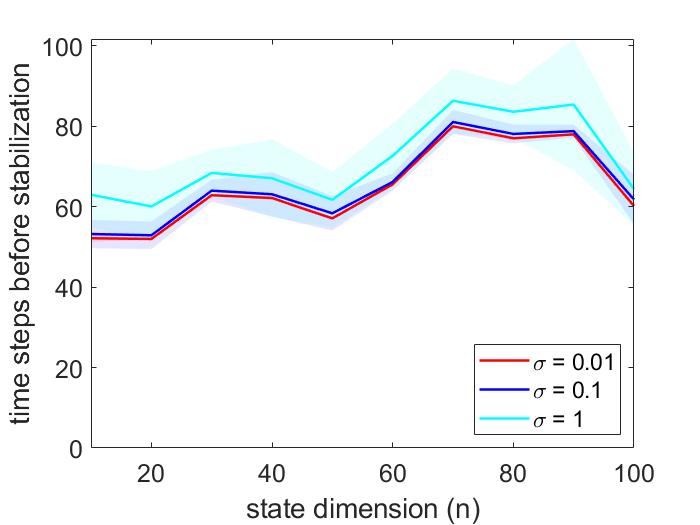}
        \caption{Running steps of LTS\textsubscript{0}-N}
        \label{fig:simulation_a}
    \end{subfigure}
    ~ 
    \begin{subfigure}[t]{0.40\textwidth}
        \centering
        \includegraphics[width=\textwidth]{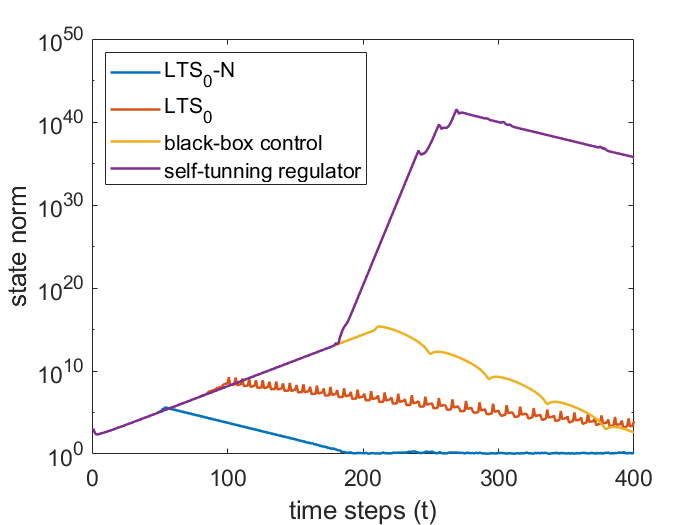}
        \caption{State norm along one trajectory}
        \label{fig:simulation_b}
    \end{subfigure}
    \caption{In (a), the line shows the average steps it takes to stabilize the system, and the shadow area shows the standard deviation. In (b), the trajectory of our algorithm, the algorithm in \citet{LTI}, the black-box controller in \citet{Chen07} and a self-turning regulator in \citet{Astrom96} are compared in a randomly generated LTI system with $n = 128$, $k = 4$, $m = 3$, and $\sigma = 0.01$.}
\end{figure*}

\textbf{Performance difference under different $n$ and $\sigma$.} \Cref{fig:simulation_a} shows the relationship between the number of steps between running LTS\textsubscript{0}-N and the dimension of states. It is evident that the increase in the number of steps is at most linear in $\log(n)$, as proven in Theorem~\ref{thm:main}. As we used the same randomly generated matrices for each $(n, \sigma)-$pair, all three curves in \Cref{fig:simulation_a} have a similar trend at each node. This observation verifies that the number of steps needed for stabilization also depends on the eigenvalue distribution of the system dynamics matrices, as we showed in the proof. Moreover, we see that an increase in noise slightly increases the number of steps for stabilization, as shown in the proof of \Cref{thm:projection}. As expected, an increase in noise also increases the standard deviation of the number of steps before stabilization. 

\textbf{Difference in performance in single trajectory} \Cref{fig:simulation_b} shows a typical trajectory of our LTS\textsubscript{0}-N algorithm. It is evident that our algorithm takes significantly fewer steps than adaptive control algorithms (self-tuning regulator and black-box control) and also fewer steps than the LTS\textsubscript{0} algorithm proposed \citet{LTI}. This is because the self-tuning regulator and the black-box control algorithm cannot take stabilzing control actions before the system runs for at least $n$ steps and learns the system dynamics. Moreover, due to the stochastic coupling of the system, estimation of system dynamics becomes much more difficult, and the adaptive control methods need a relatively large state to overcome the disturbance of noise in system identification. In comparison to LTS\textsubscript{0}, note that in this simulation, we chose $m < k$ to demonstrate the advantage of our algorithm in an under-actuated system. We see that our algorithm incurred less zig-zagging than LTS\textsubscript{0}, since we can stabilize directly on the existing state space, and LTS\textsubscript{0} has to stabilize on a composite state space, the details of which can be seen at Appendix C of \citet{LTI}. 

%
%
%



\bibliographystyle{abbrvnat} 
\bibliography{ref} 


\appendix
\section{Proof of Theorem \ref{thm:projection}}
\label{Appendix:proj_proof}
One of the key innovations of this work is the SVD-based framework we use to decouple the unstable subspace from the rest of the system. Therefore, we prove \Cref{thm:projection} here. After the system runs for time $T$, we record the state space in a $n \times T$ matrix $D$ whose $t$-th column is defined as:
\begin{equation*}
    D(t) = x_t = A x_{t-1} + \eta_{t}.
\end{equation*}
We decompose $A$ based on $E_u \oplus E_s$-decomposition. Suppose $E_u$ and $E_s$ are represented by their orthonormal bases $Q_1 \in \mathbb{R}^{n \times k}$ and $Q_2 \in \mathbb{R}^{n \times (n-k)}$, respectively, i.e. 
    $E_u = \col(Q_1), E_s = \col(Q_2) $.
Let $Q = [Q_1, Q_2]$ (which is invertible as $A$ is diagonalizable), and let $R = [R_1^*, R_2^*]^*:= Q^{-1}$. Since $E_u$ and $E_s$ are both invariant with regard to $A$, we know there exists $N_1 \in \mathbb{R}^{k \times k}$, $N_2 \in \mathbb{R}^{(n-k) \times (n-k)}$, s.t.
\begin{equation*}
    AQ = Q
    \begin{bmatrix}
    N_1 & \\ & N_2    
    \end{bmatrix}
    \Leftrightarrow
    N := \begin{bmatrix}
    N_1 & \\ & N_2    
    \end{bmatrix}
    = RAQ .
\end{equation*}

We are now ready to prove Theorem \ref{thm:projection}.

\begin{proof}
Let $D = U \Sigma V^*$ denote the compressed singular value decomposition of $D$ and $\sigma_1 > \dots > \sigma_n$ denote its singular values. In this case, we have $U \in \mathbb{R}^{n \times \min\{n,T\}}$, $\Sigma \in \mathbb{R}^{\min\{n,T\} \times \min\{n,T\}}$, and $V \in \mathbb{R}^{T \times \min\{n,T\}}$. Moreover, denote $U = [u_1,\dots,u_n]$ and $V = [v_1,\dots,v_n]$.

Furthermore, we have the following equalities
\begin{equation*}
    D = QRD = Q\begin{bmatrix}
        R_1 D \\ R_2 D
    \end{bmatrix} 
    = Q \begin{bmatrix}
        D_1 \\ D_2
    \end{bmatrix}
    = \begin{bmatrix}
    Q_1 & Q_2
\end{bmatrix}
\begin{bmatrix}
    D_1 \\ 0
\end{bmatrix}
+ 
\begin{bmatrix}
    Q_1 & Q_2
\end{bmatrix}
\begin{bmatrix}
    0 \\ D_2
\end{bmatrix}
= Q_1 D_1 + Q_2 D_2.
\end{equation*}

Let 
\begin{equation*}
    \mathcal{D} = \begin{bmatrix}
         0 & (Q_1 D_1)^* \\
         Q_1 D_1 & 0
    \end{bmatrix}
    , \quad 
    J = \begin{bmatrix}
         0 & (Q_2 D_2)^* \\
         Q_2 D_2 & 0
    \end{bmatrix}
    ,\quad 
    \mathcal{D} + J
    = \begin{bmatrix}
         0 & D^* \\
         D & 0 
    \end{bmatrix}.
\end{equation*}
We can decompose $\mathcal{D}+J$ in the following form
\begin{equation*}
    \mathcal{D} + J = \begin{bmatrix}
        0 & V \Sigma U^* \\
        U \Sigma V^* & 0
    \end{bmatrix}
    = \frac{1}{2}\left(\begin{bmatrix}
        V \\ U
    \end{bmatrix}\Sigma
    \begin{bmatrix}
        V \\ U
    \end{bmatrix}^* 
    -
    \begin{bmatrix}
        V \\ -U
    \end{bmatrix}\Sigma
    \begin{bmatrix}
        V \\ -U
    \end{bmatrix}^* \right).
\end{equation*}
Therefore, we see that the eigenvalues of $\mathcal{D}+J$ are exactly $\{\pm \sigma_i\}$ with eigenvectors $[v_i^*, \pm u_i^*]^*$, respectively. Correspondingly, the top $k$ largest eigenvalues of $\mathcal{D}+J$ are the top $k$ largest singular values of $D$, or the square root of top $k$ largest eigenvalues of $DD^*$. 

Similarly, we use compressed singular value composition on $D_1 = U_1 \Sigma_1 V_1^*$, where $U_1 \in \mathbb{R}^{k \times k}, \Sigma_1 \in \mathbb{R}^{k \times k}, V_1 \in \mathbb{R}^{T \times k}$, and decompose $\mathcal{D}$ as follows:
\begin{align*}
    \mathcal{D} =& \begin{bmatrix}
        0 & V_1 \Sigma_1 U_1^* Q_1^* \\
        Q_1 U_1 \Sigma_1 V_1^* & 0
    \end{bmatrix}
    \\
    =& \frac{1}{2}\bigg(\begin{bmatrix}
        V_1 \Sigma_1 V_1^* & V_1 \Sigma_1 U_1^* Q_1^* \\
        Q_1 U_1 \Sigma_1 V_1^* & Q_1 U_1 \Sigma_1 U_1^* Q_1^*
    \end{bmatrix}
    -
    \begin{bmatrix}
        V_1 \Sigma_1 V_1^* & -V_1 \Sigma_1 U_1^* Q_1^* \\
        -Q_1 U_1 \Sigma_1 V_1^* & Q_1 U_1 \Sigma_1 U^* Q_1^*
    \end{bmatrix}
    \bigg)
    \\
    =& \frac{1}{2}\bigg(\begin{bmatrix}
        V_1 \Sigma_1 \\ Q_1 U_1 \Sigma_1 
    \end{bmatrix}
    \begin{bmatrix}
        V_1^* & U_1^*Q_1^*
    \end{bmatrix}-
    \begin{bmatrix}
        V_1 \Sigma_1 \\ -Q_1 U_1 \Sigma_1 
    \end{bmatrix}
    \begin{bmatrix}
        V_1^* & -U_1^*Q_1^*
    \end{bmatrix}\bigg)
    \\
    =& \frac{1}{2}\left(\begin{bmatrix}
        V_1 \\ Q_1 U_1
    \end{bmatrix}\Sigma_1
    \begin{bmatrix}
       V_1 \\ Q_1 U_1
    \end{bmatrix}^* 
    -
    \begin{bmatrix}
         V_1 \\ -Q_1 U_1
    \end{bmatrix}\Sigma_1
    \begin{bmatrix}
        V_1 \\ -Q_1 U_1
    \end{bmatrix}^* \right).
\end{align*}
We see that the top $k$ largest eigenvalues of $\mathcal{D}$ are the top $k$ largest singular values of $D_1$, denoted as $\hat{\sigma}_1,\dots,\hat{\sigma}_k$. 

Let $U^{(k)}$ and $V^{(k)}$ denote the submatrices containing the first $k$ columns of $U$ and $V$, respectively. Let $\Pi $ and $\Pi'$ denote the projection onto the eigenspaces of the largest $k$ eigenvectors of $\mathcal{D}+J$ and $\mathcal{D}$, respectively.

It is clear that
\begin{align*}
    \Pi &= \frac{1}{2}
    \begin{bmatrix}
        V^{(k)} \\ U^{(k)}
    \end{bmatrix}
    \begin{bmatrix}
        (V^{(k)})^* & (U^{(k)})^*
    \end{bmatrix}
    = \frac{1}{2}
    \begin{bmatrix}
        V^{(k)} (V^{(k)})^* & V^{(k)} (U^{(k)})^*\\
        U^{(k)} (V^{(k)})^* & U^{(k)} (U^{(k)})^*
    \end{bmatrix},
    \\
    \Pi' &= \frac{1}{2}
    \begin{bmatrix}
        V_1 \\ Q_1 U_1
    \end{bmatrix}
    \begin{bmatrix}
        V_1^* & U_1^* Q_1^*
    \end{bmatrix}
    = \frac{1}{2}
    \begin{bmatrix}
        V_1 V_1^* & V_1 U_1^*Q_1^*\\
        Q_1 U_1 V_1^* & Q_1 U_1 U_1^* Q_1^*
    \end{bmatrix}
    = \frac{1}{2}
    \begin{bmatrix}
        V_1 V_1^* & V_1 U_1^*Q_1^*\\
        Q_1 U_1 V_1^* & Q_1 Q_1^*
    \end{bmatrix}.
\end{align*}

By Davis-Kahan Theorem (see \citet{Davis-Kahan} and Appendix~\ref{Appendix:DKT}), we have
\begin{equation*}
    \norm{\Pi-\Pi'} \leq \frac{1}{2}\frac{\sqrt{2k}\norm{J}_2}{\hat{\sigma}_k - \sigma_{k+1}} = \frac{\sqrt{2k}\norm{Q_2 D_2}}{\hat{\sigma}_k - \sigma_{k+1}} 
    \leq \frac{\sqrt{2k}\norm{Q_2}\norm{D_2}}{\hat{\sigma}_k - \sigma_{k+1}} = \frac{\sqrt{2k}\norm{D_2}}{\hat{\sigma}_k - \sigma_{k+1}}.
\end{equation*}
Since $
    \widehat{\Pi}_1 = U^{(k)} (U^{(k)})^*,  \Pi_1 = Q_1 Q_1^* $, we have
\begin{equation*}
    \norm{\widehat{\Pi}_1 - \Pi_1} \leq \norm{\Pi - \Pi'}\leq \frac{\sqrt{2k}\norm{D_2}}{\hat{\sigma}_k - \sigma_{k+1}}.
\end{equation*}
We next show that $\hat{\sigma}_k = \Omega(|\lambda_k|^T)$, $\sigma_{k+1} = O(T)$ and $\Vert D_2\Vert = O(T)$, based on which  $ \norm{\widehat{\Pi}_1 - \Pi_1}  \leq \frac{O(T)}{\Omega(\lambda_{k}^T - T)} \rightarrow 0$. More formally, we have the following. 



\begin{lemma}
    \label{lemm:D1_bound_final}
    If
    \begin{equation}
    \label{eqn:part_T}
    T > \Theta\left(\frac{\log k - 2 \log \left(\frac{\gap}{k^{\frac{k}{2}+3}}\right) - 3\log \theta }{\log |\lambda_k|}\right)
\end{equation}
    is satisfied, with probability at least $1-4\theta$,
    \begin{equation*}
        D_1 D_1^* \succeq \frac{\pi|\lambda_k|^{2T}\theta^2}{4} \frac{\gap^2}{k^{k+6}} \frac{|\lambda_1|^2}{|\lambda_1|^2-1} ,
    \end{equation*}
    where we recall $\gap = \left|\prod_{\substack{m_1 \neq m_2, \\ m_1, m_2 \in \{1,\dots,k\}}}(\lambda_{m_1}^{-1} - \lambda_{m_2}^{-1})\right|$.
\end{lemma}
The proof of Lemma \ref{lemm:D1_bound_final} is delayed to Appendix \ref{Appendix:D1}.

For $D_2$, we have the following inequalities
\begin{equation}
\label{eqn:D2_bound}
    \norm{D_2}_2 \leq \sqrt{T}\norm{D_2}_1 
        \leq \sqrt{T} \sum_{i = k+1}^n \left(\sum_{j = 1}^T \lambda_i^j C\right) 
        \leq \sqrt{T} (n-k)\left(\frac{C}{1-|\lambda_{k+1}|}\right) .
\end{equation}

By \Cref{lemm:D1_bound_final} and \eqref{eqn:D2_bound}, in order to have $\norm{\widehat{\Pi}_1 - \Pi_1} < \epsilon$, we need
\begin{align}
    & \norm{\widehat{\Pi}_1 - \Pi_1} < \epsilon \notag
    \\
    \Leftarrow&\frac{\sqrt{2k}\sqrt{T} (n-k)\left(\frac{C}{1-|\lambda_{k+1}|}\right)}{\frac{\sqrt{\pi}|\lambda_k|^{T}\theta}{2} \frac{\gap}{k^{\frac{k}{2}+3}} \sqrt{\frac{|\lambda_1|^2}{|\lambda_1|^2-1}} - 2\sqrt{2k}\sqrt{T} (n-k)\left(\frac{C}{1-|\lambda_{k+1}|}\right)} 
    < \epsilon
    \notag
    \\
    \Leftarrow& \frac{2\sqrt{2k}k^{\frac{k}{2}+3}\sqrt{T} (n-k)\left(\frac{C}{1-|\lambda_{k+1}|}\right)}{\sqrt{\pi}|\lambda_k|^{T}\theta \gap - 4\sqrt{2k}k^{\frac{k}{2}+3}\sqrt{T} (n-k)\left(\frac{C}{1-|\lambda_{k+1}|}\right)} 
    < \epsilon
    \notag
    \\
    \Leftarrow& \frac{2\sqrt{2}k^{\frac{k+7}{2}}\sqrt{T} (n-k)\left(\frac{C}{1-|\lambda_{k+1}|}\right)}{\frac{1}{2}\sqrt{\pi}|\lambda_k|^{T}\theta \gap} 
    < \epsilon
    \label{eqn:insertion_exp}
    \\
    \Leftarrow & 4\sqrt{2}k^{\frac{k+7}{2}}\sqrt{T} (n-k)\left(\frac{C}{1-|\lambda_{k+1}|}\right) 
    < \sqrt{\pi}|\lambda_k|^{T}\theta \gap\epsilon \notag 
    \\
    \Leftarrow & \frac{1}{2} \log T + \log\bigg(4\sqrt{2}k^{\frac{k+7}{2}} (n-k)\left(\frac{C}{1-|\lambda_{k+1}|}\right)\bigg) 
    < T \log |\lambda_k| + \log \big(\sqrt{\pi}\theta \gap\epsilon\big) \notag 
    \\
    \Leftarrow & \frac{1}{2} T \log |\lambda_k| > \log \bigg(\frac{4\sqrt{2}k^{\frac{k+7}{2}} (n-k)\left(\frac{C}{1-|\lambda_{k+1}|}\right)}{\sqrt{\pi}\theta \gap\epsilon}\bigg)
    \label{eqn:insertion2}
    \\
    \Leftarrow & T > \frac{2\log \bigg(\frac{4\sqrt{2}k^{\frac{k+7}{2}} (n-k)\left(\frac{C}{1-|\lambda_{k+1}|}\right)}{\sqrt{\pi}\theta \gap\epsilon}\bigg)}{\log |\lambda_k|}
    \label{eqn:projection_intermediate}
\end{align}
where in \eqref{eqn:insertion_exp}, we require
\begin{align}
    &4\sqrt{2k}k^{\frac{k}{2}+3}\sqrt{T} (n-k)\left(\frac{C}{1-|\lambda_{k+1}|}\right) < \frac{1}{2} \sqrt{\pi}|\lambda_k|^{T}\theta \gap 
    \notag
    \\
    \Leftarrow& \frac{1}{2}\log T + \log\bigg(4\sqrt{2}k^{\frac{k+7}{2}} (n-k)\left(\frac{C}{1-|\lambda_{k+1}|}\right)\bigg) 
    < T \log |\lambda_k| + \log(\frac{1}{2}\sqrt{\pi} \theta \gap) 
    \notag
    \\
    \Leftarrow & \frac{1}{2}T \log|\lambda_k| > \log\bigg(4\sqrt{2}k^{\frac{k+7}{2}} (n-k)\left(\frac{C}{1-|\lambda_{k+1}|}\right)\bigg) 
    - \log(\frac{1}{2}\sqrt{\pi} \theta \gap)
    \label{eqn:insertion3}
    \\
     \Leftarrow & T > \frac{2\log\bigg(\frac{8\sqrt{2}k^{\frac{k+7}{2}} (n-k)\left(\frac{C}{1-|\lambda_{k+1}|}\right)}{\sqrt{\pi} \theta \gap }\bigg)}{\log|\lambda_k|}
    \label{eqn:projection_requirement}
\end{align}
where in \eqref{eqn:insertion2} and \eqref{eqn:insertion3}, we need $T \log |\lambda_k| > \log T$.
In order to have $T \log |\lambda_k| > \log T$, define 
$$f(T) := T  \log |\lambda_k| - \log T.$$
When $T > \log |\lambda_k|$, we have $f(T) = \left(\log |\lambda_k|\right)^2 - \log \log |\lambda_k| > 0$ and $f'(T) = \log |\lambda_k| - \frac{1}{T} > 0$. 

Therefore, when $T > \log |\lambda_k|$, we have $T \log |\lambda_k| > \log T$.

Combining \eqref{eqn:projection_intermediate}, \eqref{eqn:projection_requirement}, and $T > \log |\lambda_k|$ required above, we get
\begin{equation}
\label{eqn:T_complete}
    \begin{split}
        T &> \max \Bigg\{\frac{2\log\bigg(\frac{8\sqrt{2}k^{\frac{k+7}{2}} (n-k)\left(\frac{C}{1-|\lambda_{k+1}|}\right)}{\sqrt{\pi} \theta \gap }\bigg)}{\log|\lambda_k|}, \frac{2\log \bigg(\frac{4\sqrt{2}k^{\frac{k+7}{2}} (n-k)\left(\frac{C}{1-|\lambda_{k+1}|}\right)}{\sqrt{\pi}\theta \gap\epsilon}\bigg)}{\log |\lambda_k|},  \log |\lambda_k|\Bigg\}.
    \end{split}
\end{equation}
Treating the eigenvalue terms and $\theta$ to be constants as stated in the theorem, for $\norm{\widehat{\Pi}_1 - \Pi_1} < \epsilon$ to hold, we need
\begin{equation}
\label{eqn:final_T}
    T > \Theta(\left(k \log k + \log(n-k) - \log \epsilon - \log \gap\right).
\end{equation}
This concludes the proof.
\end{proof}


\section{Auxillary Lemmas for Step 1}
\label{Appendix:D1}
We derive a lower bound on $D_1 D_1^*$ from Appendix 11 of \citet{near_optimal_LDS}, which requires two additional functions $\phi_{\min}(M_1, T)$ and $\psi(M_1,T)$: 

For the space $\mathbb{R}^d$, define the $a-$outbox, $S_d(a)$, as the following set
\begin{equation*}
    S_d(a) = \{v \in \mathbb{R}^d| \min_{1 \leq i \leq d} |v(i)| > a\},
\end{equation*}
which is used to quantify the following norm–like quantities of a matrix: 
\begin{equation}
\label{eqn:phi_defn}
    \phi_{\min}(M_1,T) = \sqrt{\inf_{v \in S_k(1)}\sigma_{\min}\left(\sum_{i=1}^T J^{-i+1}vv^* (J^{-i+1})^* \right)},
\end{equation}
where $M_1 = \Bar{P}^{-1}J \Bar{P}$ is the Jordan normal form of $M_1$. Here, Since $A$ is diagonalizable (so $M_1$ is diagonalizable), $J$ is the diagonal matrix of $\lambda_1,\cdots,\lambda_k$.

\begin{equation}
    \label{eqn:psi_defn}
    \psi(M_1,T) = \frac{1}{2k \sup_{1 \leq i \leq k} C_{|\Bar{P}_i^* z_T|}},
\end{equation}
where $C_X$ is the essential supremum of the probability distribution function (pdf) of $X$, $\Bar{P} = [\Bar{P}_1,\Bar{P}_2,\cdots,\Bar{P}_k]^*$,
and
\begin{equation}
\label{eqn:z}
    z_t := M_1^{-t} P_1^* x_{t} 
     = \sum_{j = 1}^t M_1^{-j} P_1^* \eta_j.
\end{equation}

The following lemma is adapted from Appendix 11 of \cite{near_optimal_LDS}.
\begin{lemma}
\label{lemm:D1_bound}
With probability at least $1-4\theta$,
\[
D_1 D_1^* \succeq \frac{1}{2}\phi_{\min}(M_1,T)^2 \psi(M_1,T)^2 \theta^2 M_1^T (M_1^T)^* ,
\]
whenever
\begin{equation}
\label{eqn:upsilon}
    \begin{split}
        &\left(4T^3 \lambda_k^{-2(T+1)\upsilon}k + \frac{T^2 k \sum_{i=1}^k \lambda_i^{-2(T+1)}}{\theta}\right) \\
        &\leq \frac{\phi_{\min}(M_1,T)^2 \psi(M_1,T)^2\theta^2}{2} ,
    \end{split}
\end{equation}
and
\begin{equation}
    \label{eqn:T_eigengap}
    T > \max \left\{\frac{C}{1 - |\lambda_{k+1}|}, \frac{C}{|\lambda_{k}| - 1}\right\} .
\end{equation}
for some $\upsilon$ such that $(T+1)\upsilon = \lfloor \frac{T+1}{2} \rfloor$.
\end{lemma}
Note that in \eqref{eqn:upsilon}, we select $T$ such that $\sum_{i=1}^k \sum_{t = 1}^T |\lambda_i|^{-t} < kT$, or $T > \frac{1}{k} \sum_{i=1}^k \frac{\lambda_i}{\lambda_i - 1}$. 

In Section~\ref{Appendix:Aux_D1}, we further prove the bounds on $\phi_{\min}$ and $\psi$ in Lemma~\ref{lemm:lower_bd_phi} and Lemma~\ref{lemm:lower_bd_psi},
which, combining with \Cref{lemm:D1_bound} leads to the result in \Cref{lemm:D1_bound_final} directly. It is clear that the bound in \eqref{eqn:part_T}  under \Cref{lemm:D1_bound_final} satisfies \eqref{eqn:T_eigengap} in Lemma~\ref{lemm:D1_bound} trivially. Therefore, to prove \Cref{lemm:D1_bound_final}, we just need to show that under \eqref{eqn:part_T}, \eqref{eqn:upsilon} in Lemma~\ref{lemm:D1_bound} is satisfied. 

\begin{proof}[proof of \Cref{lemm:D1_bound_final}]
To satisfy \eqref{eqn:upsilon}, we need
\begin{equation}
\label{eqn:upsilon1}
        T^3 \lambda_k^{-2(T+1)\upsilon} k
        \leq \frac{\phi_{\min}(M_1,T)^2 \psi(M_1,T)^2\theta^2}{16} ,
\end{equation}
and
\begin{equation}
    \label{eqn:upsilon2}
    \frac{T^2 k \sum_{i=1}^k |\lambda_i|^{-2(T+1)}}{\theta} \leq \frac{\phi_{\min}(M_1,T)^2 \psi(M_1,T)^2\theta^2}{4} .
\end{equation}

We then separately evaluate the conditions that would guarantee the satisfaction of the above inequities. 

\textbf{Condition~\eqref{eqn:upsilon1}}: Taking the $\log$, we have
\begin{align*}
    &3 \log T - 2(T+1)\upsilon \log |\lambda_k| + \log k \leq 2\log \left(\phi_{\min}(M_1,T) \psi(M_1,T)\theta\right) - \log 16 \\
    \stackrel{(a)}{\Leftarrow} & 3T \upsilon \log |\lambda_k| - 2(T+1)\upsilon \log |\lambda_k| + \log k \leq 2\log \left(\phi_{\min}(M_1,T) \psi(M_1,T)\theta\right) - \log 16 \\
    \Leftarrow& -(3T+2) \upsilon \log |\lambda_k| \leq 2\log \left(\phi_{\min}(M_1,T) \psi(M_1,T)\theta\right) - \log 16 - \log k \\
    \Leftarrow& T \geq \frac{\log 16 + \log k - 2\log \left( \phi_{\min}(M_1,T) \psi(M_1,T)\theta\right)}{3\upsilon \log |\lambda_k|}  - 2 .
\end{align*}
where the step $(a)$ uses the following: $T \upsilon \log |\lambda_k| > \log T$, which we show now. Define
\begin{equation*}
    f(T) := T \log |\lambda_k|^\upsilon - \log T .
\end{equation*}
When $T = \log |\lambda_k|^\upsilon$, we have $f(T) = \left(\log |\lambda_k|^\upsilon\right)^2 - \log \log |\lambda_k|^{\upsilon} > 0$. When $T \geq \log |\lambda_k|^\upsilon$, we have $f'(T) = \log |\lambda_k|^\upsilon - \frac{1}{T} > 0$.

Therefore, when $T > \log |\lambda_k|^\upsilon$, we have $T \upsilon \log |\lambda_k| > \log T$. 

\textbf{Condition~\eqref{eqn:upsilon2}}: Since $|\lambda_1|>\ldots> |\lambda_k|$, to meet \eqref{eqn:upsilon2}, it suffices to show: 
\begin{align*}
    &T^2 k^2 |\lambda_k|^{-2(T+1)} \leq \frac{\phi_{\min}(M_1,T)^2 \psi(M_1,T)^2 \theta^3}{4} 
    \\
    \Leftarrow &T^2 |\lambda_k|^{-2(T+1)} \leq \frac{\phi_{\min}(M_1,T)^2 \psi(M_1,T)^2 \theta^3}{4 k^2}
    \\
    \Leftarrow & 2\log T - 2(T+1) \log |\lambda_k| \leq \log \frac{\phi_{\min}(M_1,T)^2 \psi(M_1,T)^2 \theta^3}{4 k^2}  
    \\
    \Leftarrow & T \log |\lambda_k| - 2(T+1) \log |\lambda_k| \leq \log \frac{\phi_{\min}(M_1,T)^2 \psi(M_1,T)^2 \theta^3}{4 k^2}
    \\
    \Leftarrow & T \geq -\frac{\log \frac{\phi_{\min}(M_1,T)^2 \psi(M_1,T)^2 \theta^3}{4 k^2}}{\log |\lambda_k|} + 2 .
\end{align*}
Similar to the derivation of \eqref{eqn:upsilon1}, in order to get $T \log |\lambda_k| > 2 \log T$, we need $T > 2 \log |\lambda_k|$.

Combining the above and applying \Cref{lemm:lower_bd_phi} and \Cref{lemm:lower_bd_psi}, we get the condition for $T$ as in \eqref{eqn:part_T}.

This concludes the proof of \Cref{lemm:D1_bound_final}.
\end{proof}

The following Corollary directly follows from \Cref{thm:projection}.
\begin{corollary}
    \label{coro:52.1}
    Under the premise of Theorem~\ref{thm:projection}, for any orthonormal basis $\hat{P}_1$ of $\col(\widehat{\Pi}_1)$ (where $\widehat{\Pi}_1$ is obtained by Algorithm~\ref{alg:LTS0}), there exists a corresponding orthonormal basis $P_1$ of $\col(\Pi_1)$, such that
    \begin{equation*}
        \norm{P_1 - \hat{P}_1} < \sqrt{2k}\epsilon := \delta.
    \end{equation*}
\end{corollary}
The proof structure of \Cref{coro:52.1} is identical to the proof of Corollary 5.2 of \cite{LTI}. 
\section{Proof of Auxiliary Lemmas for Appendix~\ref{Appendix:D1}}
\label{Appendix:Aux_D1}
In this section, we prove a few Lemmas that is used to bound $D_1 D_1^*$ in \Cref{Appendix:D1}. 
\begin{lemma}
\label{lemm:vandermonde}
Given a $k\times k$ Vandermonde Matrix $\Lambda$
\begin{equation}
    \label{Vandermonde}
   \Lambda = \begin{bmatrix}
    1 & \lambda_1^{-1} & \cdots & \lambda_1^{-k+1}
    \\ 
    \vdots & \vdots & \vdots & \vdots 
    \\ 
    1 & \lambda_k^{-1} & \cdots & \lambda_k^{-k+1}
\end{bmatrix},
\end{equation}
and $\lambda_1, \ldots, \lambda_k \neq 0$, then $\norm{\Lambda^{-1}} \leq \frac{k^{\frac{k}{2}+1}}{\gap}$,
where 
\begin{equation}
\label{eqn:eigen_gap}
    \gap = \left|\prod_{m_1 \neq m_2}(\lambda_{m_1}^{-1} - \lambda_{m_2}^{-1})\right| .
\end{equation}
\end{lemma}
\begin{proof}
From Theorem 1 of \cite{eig}, we have 
\begin{equation}
\label{eqn:Lambdaij}
    \Lambda^{-1}(i,j) = \frac{(-1)^{k-i} \sigma^j_{k-i}}{\prod_{m_1 \neq m_2}(\lambda_{m_1}^{-1} - \lambda_{m_2}^{-1})},
\end{equation}
where $\sigma_{k-i}^j := \sum_{\rho_{k-i}^j} \prod_{\ell \in \rho_{k-i}^j} \lambda_{\ell}^{-1}$ and $\rho_{k-i}^j$ goes through all subsets of $\{\lambda_1^{-1}, \ldots, \lambda_{j-1}^{-1}, \lambda_{j+1}^{-1},\ldots, \lambda_k^{-1}\}$ with cardinality $k-i$. In the above expression, the quantity $\sigma_{k-i}^j$ can be bounded as:
\begin{equation}
    \label{eqn:bdd_sigma}
    \sigma_{k-i}^j \leq \binom{k}{k-i}\left(\frac{1}{\lambda_k}\right)^{k-i}.
\end{equation}
Plugging \eqref{eqn:bdd_sigma} into \eqref{eqn:Lambdaij} gives a bound for $|\Lambda^{-1}(i,j)|$ as follows:
\begin{equation}
    \label{eqn:Lambdaij_bound}
    \norm{\Lambda^{-1}(i,j)} \leq \frac{\binom{k}{k-i}\left(\frac{1}{\lambda_k}\right)^{k-i}}{\gap}.
\end{equation}
Moreover, we have the following well-known inequality (see, for example, \citet{Horn85}) 
\begin{equation}
    \frac{1}{\sqrt{k}}\norm{\Lambda^{-1}}_1 \leq \norm{\Lambda^{-1}}_2 \leq \sqrt{k}\norm{\Lambda^{-1}}_1.
\end{equation}
Combining the above, we get
\begin{equation}
    \norm{\Lambda^{-1}} 
    \leq \max_i\left\{ \sum_{j}\left|\Lambda^{-1}(i,j)\right|\right\}
    \leq \frac{k^{\frac{k}{2}+\frac{3}{2}}}{\gap}.
\end{equation}
where we have used the Sterling's formula for bounding $\binom{k}{k-i}$ in the summation.
\end{proof}

\begin{lemma}
\label{lemm:lower_bd_phi}
Under the premise of \Cref{thm:projection}, given $\phi_{\min}$ as defined in \eqref{eqn:phi_defn}, we have
\begin{equation*}
    \phi_{\min}(M_1,T) \geq \frac{\gap}{k^{\frac{k}{2}+2}}.
\end{equation*}
\end{lemma}

\begin{proof}

Let $h_i(v) = \begin{bmatrix}
          \lambda_1^{-i+1} v(1) \\ \lambda_2^{-i+1} v(2) \\ \vdots \\ \lambda_k^{-i+1} v(k)
      \end{bmatrix} \in \mathbb{R}^k$, and
      $H(v) = \begin{pmatrix}
          h_1(v) & h_2(v) & \dots & h_T(v)
      \end{pmatrix}$.
      Then we have
\[
\displaystyle
\begin{array}{rcl}
     \phi_{\min}(M_1,T) &=& \sqrt{\inf_{v \in S_d(1)} \sigma_{\min} \left(\sum_{i = 1}^T
      h_i(v) h_i^*(v)
     \right)}\\
     &=& \sqrt{\inf_{v \in S_d(1)} \sigma_{\min} \left(H(v) H^*(v)\right)}\\
     &=& \sqrt{\inf_{v \in S_d(1)} \frac{1}{\norm{H^{-1}(v)}^2}}\\
     &=& \inf_{v \in S_d(1)} \frac{1}{\norm{H^{-1}(v)}}
\end{array}
\]
and we can decompose $H(v)$ as follows 
\begin{equation*}
    \begin{split}
        H(v) &= \text{diag}\left(v(1), \dots, v(k)\right)
    \begin{bmatrix}
        1 & \lambda_1^{-1} & \dots & \lambda_1^{-T+1}\\
        1 & \lambda_2^{-1} & \dots & \lambda_2^{-T+1}\\
        \vdots & \vdots & \vdots & \vdots\\
        1 & \lambda_k^{-1} & \dots & \lambda_k^{-T+1}
    \end{bmatrix}\\
    &:= \text{diag}(v) \Tilde{H}.
    \end{split}
\end{equation*}
Therefore, 
\begin{equation}
\label{eqn:ineq_Yv}
    \norm{H^{-1}(v)} = \norm{\Tilde{H}^{-1} (\text{diag}(v))^{-1}} \leq \norm{\Tilde{H}^{-1}} \norm{(\text{diag}(v))^{-1}}.
\end{equation}
By Lemma \ref{lemm:vandermonde}, we get
\begin{equation*}
    \norm{\Tilde{H}^{-1}} \leq \frac{k^{\frac{k}{2}+\frac{3}{2}}}{\gap}.
\end{equation*}

Plugging the above inequality into \eqref{eqn:ineq_Yv} gives 
\begin{equation*}
    \norm{H^{-1}(v)} \leq \frac{k^{\frac{k}{2}+2}}{\gap},
\end{equation*}
and
\begin{equation*}
    \phi_{\min}(M_1,T) \geq \frac{\gap}{k^{\frac{k}{2}+2}}.
\end{equation*}
\end{proof}

We also need a similar bound for $\psi(M_1,T)$. However, since we do not have an explicit formula for the pdf of noise $\eta$, it is difficult to evaluate $\sup_{1 \leq i \leq k}C_{|\Bar{P}_i^* z_T|}$ in \eqref{eqn:psi_defn} explicitly. However, it is intuitively clear that $\sup_{1 \leq i \leq k}C_{|\Bar{P}_i^* z_T|}$ is upper bounded by a constant, as $z_T$ in \eqref{eqn:z} converges in distribution as $T \rightarrow \infty$. Therefore, the probability distribution function of $\Bar{P}_i^* z_T$ also converges. 

To demonstrate this more concretely, we explicitly compute the bound when $\eta_t \sim N(0,1)$ follows the standard normal distribution:
\begin{lemma}
    \label{lemm:upper_bd_Cz}
    If $\eta_t$ follows the standard normal distribution for all $t$, then
    \begin{equation*}
        C_{|\Bar{P}_i^* z_T|} < \sqrt{\frac{2}{\pi}} \sqrt{\frac{|\lambda_i|^2-1}{|\lambda_i|^2}}
    \end{equation*}
\end{lemma}
\begin{proof}
    The $j$-th entry of $\Bar{P}_i^* z_T$ can be represented as 
\[
\Bar{P}_i^* z_{T} (j)= \sum_{t = 1}^T v_i \cdot \left(M_1^{-t} P_1^* \eta_t\right) \sim N\left(0, \sum_{t = 1}^T \left(|\lambda_j|^{-t}\right)^2\right)
\]
so 
$$\text{pdf}_{\Bar{P}_i^* Z_{T}}(y) = \frac{1}{\sqrt{2\pi \sum_{t = 1}^T \left(|\lambda_i|^{-t}\right)^2}}e^{-\frac{y^2}{2\sum_{t = 1}^T \left(|\lambda_i|^{-t}\right)^2}}, \qquad y \in \mathbb{R}.$$
With some algebra, we get 
$$\text{pdf}_{|\Bar{P}_i^* z_{T}|}(y) = \frac{\sqrt{2}}{\sqrt{\pi \sum_{t = 1}^T \left(|\lambda_i|^{-t}\right)^2}}e^{-\frac{y^2}{2\sum_{t = 1}^T \left(|\lambda_i|^{-t}\right)^2}}, \qquad y \in \mathbb{R}^+.$$
Therefore, $C_{|\Bar{P}_i^* z_T|} \leq \frac{\sqrt{2}}{\sqrt{\pi \sum_{t = 1}^T \left(|\lambda_i|^{-t}\right)^2}} \leq \sqrt{\frac{2}{\pi}} \sqrt{\frac{|\lambda_i|^2-1}{|\lambda_i|^2}}$.
\end{proof}

In the rest of the paper, we will assume $C_{|\Bar{P}_i'z_T|}$ is bounded and take 
\begin{equation}
\label{eqn:Cz}
    C_{|\Bar{P}_i^* z_T|} < C_z, 
\end{equation}
for some constant $C_z$, as in \Cref{assumption:pdf}. Therefore, the following result directly follows:
\begin{lemma}
    Under the premise of \Cref{thm:projection}, given $\psi$ as defined in \eqref{eqn:psi_defn}, we have
    \label{lemm:lower_bd_psi}
    \begin{equation*}
        \psi(M_1,T) \geq \frac{1}{2k C_z}.
    \end{equation*}
\end{lemma}

\section{Solution to the Least Square Problem in Stage 2}
\label{Appendix:ls}
Lemma~\ref{lemm:ls} gives the explicit form for the solution to the least squares problem in Algorithm~\ref{alg:LTS0}

\begin{lemma}
\label{lemm:ls}
Given $D := [x_0, \cdots, x_T]$ and $\hat{\Pi}_1 = U^{(k)} (U^{(k)})^*$, the solution to 
\begin{equation*}
    \hat{M}_1 = \arg\min_{M_1} \sum_{t=0}^T \norm{(U^{(k)})^* x_{t+1} - M_1 (U^{(k)})^* x_t}^2
\end{equation*}
is uniquely given by $\hat{M}_1 = (U^{(k)})^* A U^{(k)} + \varpi$, where $\varpi = \left(\sum_t (U^{(k)})^* \eta_t x_t^* U^{(k)}\right)((\Sigma^{(k)})^2)^{-1}$.
\end{lemma}
\begin{proof}
    Sincec $M_1$ is a stationary point of $\mathcal{L}$, for any $\Delta$ in the neighborhood of $O$, we have
    \begin{align*}
        0 \leq& \mathcal{L}(M_1 + \Delta) - \mathcal{L}(M_1)\\
        =& \sum_t \norm{\hat{y}_{1,t+1} - M_1 \hat{y}_{1,t} - \Delta \hat{y}_{1,t}}^2 - \sum_t \norm{\hat{y}_{1,t+1} - M_1 \hat{y}_{1,t}}^2 
        \\
        =& \sum_t \langle \Delta \hat{y}_{1,t}, \hat{y}_{1,t+1} - M_1 \hat{y}_{1,t} \rangle + O(\norm{\Delta}^2) 
        \\
        =& \sum_t \text{tr}\left(\hat{y}_{1,t}^* \Delta^*(\hat{y}_{1,t+1}-M_1\hat{y}_{1,t})\right) + O(\norm{\Delta}^2)
        \\
        =& \sum_t \text{tr} \left(\Delta^*(\hat{y}_{1,t+1}-M_1 \hat{y}_{1,t})\hat{y}_{1,t}^*\right) + O(\norm{\Delta}^2)
        \\
        =& \text{tr}\left(\Delta^* \sum_t \left(\hat{y}_{1,t+1} - M_1\hat{y}_{1,t}\right)\hat{y}_{1,t}^*\right) + O(\norm{\Delta}^2).
    \end{align*}
    Since the above holds for all $\Delta$, we get
    \small
    \begin{equation*}
        \sum_t (\hat{y}_{1,t+1} - M_1 \hat{y}_{1,t})\hat{y}_{1,t}^* \Leftrightarrow M_1 \sum_t \hat{y}_{1,t} \hat{y}_{1,t}^* = \sum_t \hat{y}_{1,t+1} \hat{y}_{1,t}^*.
    \end{equation*}
    \normalsize
    Plugging in $\hat{y}_{1,t} = (U^{(k)})^* x_t$ and $\hat{y}_{1,t+1} = (U^{(k)})^* (Ax_t + \eta_t)$, we have 
    \begin{equation*}
        \begin{split}
            M_1 (U^{(k)})^* DD^* U^{(k)} &= M_1 \sum_t (U^{(k)})^* x_t x_t^* U^{(k)}
            \\
            &= \sum_t (U^{(k)})^* (Ax_t + \eta_t) x_t^* U^{(k)} 
            \\
            &= (U^{(k)})^* A DD^* U^{(k)} + \sum_t (U^{(k)})^* \eta_t x_t^* U^{(k)}.
        \end{split}
    \end{equation*}
    Since $U^{(k)}$ are the first $k$ singular vectors of $D$, we have the following equalities:
    \begin{equation}
    \label{eqn:inverse_term}
        (U^{(k)})^* DD^* U^{(k)} = (U^{(k)})^* U \Sigma V^* V \Sigma^* U^* U^{(k)} = \begin{bmatrix}
                I^{(k)} & 0
            \end{bmatrix} \Sigma^2 \begin{bmatrix}
                I^{(k)} \\ 0
            \end{bmatrix} = (\Sigma^{(k)})^2,
    \end{equation}
    which is invertible, and $\hat{M}_1$ is explicitly given by
    \small
    \begin{equation}
    \label{eqn:M1hat_interm}
        \hat{M}_1 = \left((U^{(k)})^* A DD^* U^{(k)} + \sum_t (U^{(k)})^* \eta_t x_t^* U^{(k)}\right)(\Sigma^{(k)})^{-2}.
    \end{equation}
    \normalsize
    Moreover, we have
    \begin{align*}
        &U^{(k)} (U^{(k)})^* DD^*U^{(k)} =  U^{(k)} (\Sigma^{(k)})^2
        \\
        =& \begin{bmatrix}
            U^{(k)} & 0 
        \end{bmatrix}
        \begin{bmatrix}
            (\Sigma^{(k)})^2 \\ 0
        \end{bmatrix}
        = U \begin{bmatrix}
            (\Sigma^{(k)})^2 \\ 0
        \end{bmatrix}
        \\
        =& U \Sigma^2 \begin{bmatrix}
            I^{(k)} \\ 0
        \end{bmatrix}
        = U \Sigma^2 U^* U^{(k)} = D D^* U^{(k)},
    \end{align*}
    where the first equality is obtained by using \eqref{eqn:inverse_term}. Substituting the above in \eqref{eqn:M1hat_interm} yields
    \small
    \begin{equation*}
        \begin{split}
            \hat{M}_1 &= \left((U^{(k)})^* A(U^{(k)} (U^{(k)})^* DD^*)U^{(k)}\right)(\Sigma^{(k)})^{-2} + \varpi
            \\
            &=\left((U^{(k)})^* A U^{(k)} (U^{(k)})^*\right) \left( DD^* U^{(k)}\right)(\Sigma^{(k)})^{-2}+ \varpi
            \\
            &= (U^{(k)})^* A U^{(k)} + \varpi,
        \end{split}
    \end{equation*}
    \normalsize
    where $\varpi = \left(\sum_t (U^{(k)})^* \eta_t x_t^* U^{(k)}\right)(\Sigma^{(k)})^{-2}$.
\end{proof}

We want to show $(U^{(k)})^* A U^{(k)}$ is the dominating term of the above expression, as we will bound $\varpi$ in the following lemma.
\begin{lemma}
\label{lemm:52.2}
Under the premise of \Cref{thm:projection},
    \begin{equation*}
        \norm{M_1 - \hat{M}_1} < 3\norm{A} \delta
    \end{equation*}
    for any $\delta > 0$ whenever 
    \begin{equation*}
        T \geq \frac{\log \left(\frac{4C}{\pi \theta^2 \norm{A}\delta} \frac{k^{k+6}}{\gap^2}  \right)}{\log |\lambda_k|}.
    \end{equation*}
\end{lemma}
\begin{proof}
    First, we prove that $\varpi \leq \delta$. 
    Let $H = [\eta_1,\dots,\eta_T]$, then we have 
    \begin{align*}
        \varpi &= (U^{(k)})^* H D^* U^{(k)} (\Sigma^{(k)})^{-2}
        \\
        &= (U^{(k)})^* H V \Sigma^* U^* U^{(k)} (\Sigma^{(k)})^{-2}
        \\
        &= (U^{(k)})^* H V \Sigma^* \begin{bmatrix}
            I^{(k)} \\ 0
        \end{bmatrix} (\Sigma^{(k)})^{-2}
        \\
        &= (U^{(k)})^* H V \begin{bmatrix}
            \Sigma^{(k)} \\ 0
        \end{bmatrix} (\Sigma^{(k)})^{-2}
        \\
        &= (U^{(k)})^* H V \begin{bmatrix}
            (\Sigma^{(k)})^{-1} \\ 0
        \end{bmatrix}.
    \end{align*}
    Therefore, 
    \begin{align}
        &\norm{\varpi} \leq \norm{A}\delta \notag
        \\
        \Leftarrow& \norm{H} \norm{(\Sigma^{(k)})^{-1}} \leq \norm{A}\delta \notag
        \\
        \Leftarrow& \sqrt{T} C \frac{2}{\sqrt{\pi}|\lambda_k|^{T}\theta} \frac{k^{\frac{k}{2}+3}}{\gap} \sqrt{\frac{|\lambda_1|^2-1}{|\lambda_1|^2}} \leq \norm{A}\delta 
        \label{eqn:bibbers_ineq}
        \\
        \Leftarrow & \frac{|\lambda_k|^{T}}{\sqrt{T}} \geq \frac{2C}{\sqrt{\pi} \theta \norm{A}\delta} \frac{k^{\frac{k}{2}+3}}{\gap} 
        \notag
        \\
        \Leftarrow & T \log |\lambda_k| - \frac{1}{2}\log T \geq \log \left(\frac{2C}{\pi \theta \norm{A}\delta} \frac{k^{\frac{k}{2}+3}}{\gap}  \right)
        \notag
        \\
        \Leftarrow & \frac{1}{2} T \log |\lambda_k| \geq \log \left(\frac{2C}{\pi \theta \norm{A}\delta} \frac{k^{\frac{k}{2}+3}}{\gap}  \right)
        \label{eqn:repeat_insert}
        \\
        \Leftarrow & T \geq \frac{2\log \left(\frac{2C}{\pi \theta \norm{A}\delta} \frac{k^{\frac{k}{2}+3}}{\gap}  \right)}{\log |\lambda_k|},
        \label{eqn:T_additional_criteria}
    \end{align}
    where \eqref{eqn:bibbers_ineq} used \Cref{lemm:D1_bound_final} and that for a $n \times T$ matrix $H$, $\norm{H}_2 \leq \sqrt{T} \norm{H}_1$, and \eqref{eqn:repeat_insert} requires $\log T < T \log |\lambda_k|$, which is satisfied when we derived \eqref{eqn:insertion2} and \eqref{eqn:insertion3}. We can use \Cref{lemm:D1_bound_final} to bound $\norm{(\Sigma^{(k)})^{-1}}$ is a direct result of Cauchy Interlacing Theorem. We further observe that \eqref{eqn:T_additional_criteria} does not change the criteria obtained in \eqref{eqn:final_T}. 

    Recall that $U^{(k)} = \hat{P}_1$. We obtain
    \begin{align*}
        \norm{M_1 - \hat{M}_1} &= P_1^* A P_1 - \left((U^{(k)})^* A U^{(k)} + \varpi\right)
        \\
        & \leq \norm{P_1^* A P_1 - P_1^* A \hat{P}_1^*} + \norm{P_1^* A \hat{P}_1 - \hat{P}_1^* A \hat{P}_1^*} + \norm{\varpi}
        \\
        & \leq \norm{A}\norm{P_1 - \hat{P}_1} + \norm{A}\norm{P_1 - \hat{P}_1} + \norm{\varpi}
        \\
        &\leq 3 \norm{A}\delta.
    \end{align*}
    where in the last inequality, we used \Cref{coro:52.1}.
\end{proof}

With \Cref{lemm:52.2}, we are ready to prove \Cref{prop:G2}.

\begin{proof}[Proof of \Cref{prop:G2}]
    By \Cref{lemm:52.2}, we get $\norm{M_1 - \hat{M}_1} < 3\norm{A} \delta$. Moreover, by Gelfand's formula, we have
    \begin{align*}
        \norm{M_1^t} &= \norm{P_1^* A^t P_1} \leq \norm{A^t} \leq \zeta_{\epsilon_1}(A)(|\lambda_1| + \epsilon_1)^t, 
        \\
        \norm{\hat{M}_1^t} &= \norm{\hat{P}_1^* A^t \hat{P}_1} \leq \norm{A^t} \leq \zeta_{\epsilon_1}(A)(|\lambda_1| + \epsilon_1)^t, 
    \end{align*}
    Therefore, by telescoping, we get
    \begin{align*}
        \norm{M_1^{\tau} - \hat{M}_1^{\tau}} &= \norm{\sum_{i=1}^{\tau}(M_1^i \hat{M}_1^{\tau-i} - M_1^{i-1}\hat{M}_1^{\tau-i+1})}
        \\
        &\leq \norm{M_1^{i-1}}\norm{M_1^{\tau-i}}\norm{M_1 - \hat{M}_1} 
        \\
        &< \tau \cdot \zeta_{\epsilon_1}(A)^2 (|\lambda_1| + \epsilon_1)^{\tau-1}\cdot 3\norm{A}\delta
        \\
        &= 3 \tau \norm{A} \zeta_{\epsilon_1} (A)^2(|\lambda_1|+\epsilon_1)^{\tau-1} \delta.
    \end{align*}
\end{proof}

With \Cref{prop:G2}, the following corollary easily follows:
\begin{corollary}
    \label{coro:G2}
    Under the premise of \Cref{thm:main}, when $\delta < \frac{1}{\tau}$,
    \begin{equation*}
        \norm{\hat{M}_1^\tau} < \left(\zeta_{\epsilon_1}(M_1)(|\lambda_1| + \epsilon_1) + 3 \norm{A}\zeta_{\epsilon_1}(A)\right)(|\lambda_1| + \epsilon_1)^{\tau-1}.
    \end{equation*}
\end{corollary}
\begin{proof}
    By Gelfand's formula and \Cref{prop:G2}, 
    \begin{align*}
        \norm{\hat{M}_1^{\tau}} &\leq \norm{M_1^{\tau}} + \norm{\hat{M}_1^{\tau} - M_1^{\tau}} 
        \\
        &\leq \zeta_{\epsilon_1}(A)(\lambda_1 + \epsilon_1)^{\tau} + 3 \tau \norm{A} \zeta_{\epsilon_1} (A)^2(|\lambda_1|+\epsilon_1)^{\tau-1} \delta
        \\
        &< \left(\zeta_{\epsilon_1}(M_1)(|\lambda_1| + \epsilon_1) + 3 \norm{A}\zeta_{\epsilon_1}(A)\right)(|\lambda_1| + \epsilon_1)^{\tau-1}.
    \end{align*}
    where the last inequality requires $\delta < \frac{1}{\tau}$. 
\end{proof}

\section{Bounding $\norm{\hat{B}_\tau - B_\tau}$}
\label{Appendix:boundingB}
\begin{lemma}
    \label{lemm:ST_equiv_rev} 
    For any $\gamma > \epsilon$, the following implication holds:
    \begin{equation*}
        \frac{\norm{R_2 x}}{\norm{x}} \leq \gamma - \epsilon := \gamma' \quad \Rightarrow \quad \frac{\norm{(I - \hat{\Pi}_1)x}}{\norm{x}} \leq \gamma
    \end{equation*}
\end{lemma}
\begin{proof}
    \begin{align}
        \frac{\norm{(I - \hat{\Pi}_1)x}}{\norm{x}} &= \frac{\norm{(I - \hat{\Pi}_1 + \Pi_1 - \Pi_1)x}}{\norm{x}} \notag
        \\
        \leq & \frac{\norm{(I - \Pi_1) x}}{\norm{x}} + \frac{\norm{\hat{\Pi}_1 - \Pi_1}\norm{x}}{\norm{x}}
        \notag
        \\
        \leq& \frac{\norm{\Pi_2 x}}{\norm{x}} + \epsilon
        \notag
        \\
        =& \frac{\norm{\Pi_2 \Pi_s x}}{\norm{x}} + \epsilon
        \label{eqn:B2.1}
        \\
        \leq& \frac{\norm{\Pi_s x}}{\norm{x}} + \epsilon
        \notag 
        \\
        \leq & \gamma
    \end{align}
where \eqref{eqn:B2.1} holds because $E_2$ is orthogonal to $E_1$, therefore $\Pi_2 \Pi_u = 0$, as $P_2P_2^* Q_1 R_1=0$ by orthogonality of $P_2$ and $Q_1 = P_1$.
\end{proof}

In the following propositions, we show that the stopping time $\omega_i$ defined in Algorithm~\ref{alg:LTS0} guarantees a bound on $\norm{x_t}$.

\begin{proposition}
\label{prop:base_portion}
    Under the premise of Theorem~\ref{thm:main}, for any constant $\gamma > \epsilon$, if in the open loop system, 
    \begin{equation*}
        \frac{\norm{(I - \hat{\Pi}_1) x_{t}}}{\norm{x_{t}}} > \gamma,
    \end{equation*}
    then, exists $C_{\gamma} \in \mathbb{R}^+$ such that $\norm{x_{t}} < C_{\gamma}$. 
\end{proposition}

\begin{proof}
 Since we have that $x_{t} = \sum_{j = 0}^{t} A^{t-j}\eta_{j}$, we have
    \small
    \begin{equation*}
        R x_{t} = \begin{bmatrix}
            R_1 x_{t} \\ R_2 x_{t}
        \end{bmatrix}
        =
        \begin{bmatrix}
            R_1 \sum_{j = 0}^{t} A^{t-j}\eta_j \\ R_2 \sum_{j = 0}^{t} A^{t-j}\eta_j
        \end{bmatrix}
        = 
        \begin{bmatrix}
            \sum_{j = 0}^{t} N_1^{t-j} R_ 1\eta_j \\ \sum_{j = 0}^{t} N_2^{t-j} R_2\eta_j
        \end{bmatrix}
    \end{equation*}
    \normalsize
    Therefore, we have that 
    \small
    \begin{align*}
        \norm{R_2 x} &\leq \sum_{j=0}^{t} \norm{N_2^j} \norm{R_2} C \leq \sum_{j=0}^{t}\zeta_{\epsilon_4}(N_2)(\lambda_{k+1}+\epsilon_4)^j \norm{R_2} C
        \\
        &\leq \frac{\zeta_{\epsilon_4}(N_2) C}{1-\xi} \frac{1}{1 - (\lambda_{k+1}+\epsilon_4)}
    \end{align*} 
    \normalsize
    where we used Lemma A.1 of \citet{LTI}. As $\norm{R_2 x_{j}}$ is bounded above by a constant, so is $\norm{\Pi_s x_t} = \norm{Q_2 R_2 x_t}$. 
    
    Since $\frac{\norm{(I - \hat{\Pi}_1) x_{t}}}{\norm{x_{t}}} > \gamma$, by Lemma~\ref{lemm:ST_equiv_rev}, $\frac{\norm{R_2 x_t}}{\norm{x_t}} > \gamma'$. 
    Correspondingly, we have
    \begin{align*}
        \gamma' < \frac{\norm{R_2 x_t}}{\norm{x_t}},
    \end{align*}
    which implies
    \begin{equation}
    \label{eqn:C_gamma}
        \norm{x_t} < \frac{\zeta_{\epsilon_4}(N_2) C}{\gamma'(1-\xi)} \frac{1}{1 - (|\lambda_{k+1}|+\epsilon_4)} := C_{\gamma}.
    \end{equation}
\end{proof}

\begin{proposition}
\label{prop:ST_induction}
    Under the premise of Theorem~\ref{thm:main}, for any constant $\gamma > \epsilon$, consider the initial state $x_i$ such that $\frac{\norm{P_2^* x_i}}{\norm{x_i}} > \gamma$. Moreover, $x_{i+1} = A x_{i} + Bu + \eta_i$, i.e. we insert control right after the initial state and let the system run in open-loop thereafter. If for $t \in \mathbb{Z}^+$ such that
    \begin{equation*}
        \frac{\norm{(I - \hat{\Pi}_1) x_{i+t}}}{\norm{x_{i+t}}} > \gamma,
    \end{equation*}
    then, for all $\alpha < \frac{1}{\norm{B}}$,
    \begin{equation*}
        \norm{x_{i+t}} < \frac{1}{\gamma'} \left(\frac{2\zeta_{\epsilon_4}(N_2)}{1-\xi}\norm{x_i} + C_{\gamma}\right).
    \end{equation*}
\end{proposition}
\begin{proof}
    \begin{align}
        \norm{R_2 x_{i+t}} \leq& \norm{N_2^t R_2 x_i + N_2^{t-1}R_2 B u} + \sum_{j=0}^t \norm{N_2^j} \norm{R_2} C\notag
        \\
        \leq& \frac{\zeta_{\epsilon_4}(N_2)}{1-\xi}(|\lambda_{k+1}|+\epsilon_4)^{t-1}((1 + \alpha\norm{B})\norm{x_i}) + C_{\gamma}\notag
        \\
        \leq& \frac{2\zeta_{\epsilon_4}(N_2)}{1-\xi}(|\lambda_{k+1}|+\epsilon_4)^{t-1}\norm{x_i} + C_{\gamma}.
        \label{eqn:propv4.1}
    \end{align}
    Since $\frac{\norm{(I - \hat{\Pi}_1) x_{i+t}}}{\norm{x_{i+t}}} > \gamma$, by Lemma~\ref{lemm:ST_equiv_rev}, we have that 
    \begin{align*}
        \gamma' < \frac{\norm{R_2 x_{i+t}}}{\norm{x_{i+t}}}.
    \end{align*}
    Substitute the above in \eqref{eqn:propv4.1} finishes the proof.
\end{proof}

\begin{proposition}
\label{prop:ST_final}
    Under the premise of Theorem~\ref{thm:main}, for any constant $\gamma > \epsilon$ and stopping time $\omega_{i}$ such that:
    \begin{equation*}
        \omega_i = \min \left\{t > t_{i-1} : \frac{\norm{(I - \hat{\Pi}_1) x_{t}}}{\norm{x_{t}}} \leq \gamma \wedge  \norm{x_t} > \frac{C}{\delta}\right\},
    \end{equation*}
    where we assume $t_0 = T$. Then, Algorithm~\ref{alg:LTS0} guarantees that 
    \begin{equation*}
        \frac{\norm{P_2^* x_{t_i}}}{\norm{x_{t_i}}} < \gamma + \epsilon, \qquad \forall i \in \{1,\dots, m\},
    \end{equation*}
    while maintaining 
    \begin{equation*}
        \norm{x_{t_1}} \leq \max\left\{\norm{A}\frac{C}{\delta} + C, \norm{A} C_{\gamma} + C, \norm{x_T}\right\},
    \end{equation*}
    \begin{equation*}
        \norm{x_{t}} < \max\left\{\norm{A}\frac{C}{\delta} + C, \left(\frac{\norm{A}}{\gamma'}\frac{2\zeta_{\epsilon_4}(N_2)}{1-\xi}\right)^i \norm{x_{t_1}} + \sum_{j=1}^{i-1} \left(\frac{\norm{A}}{\gamma'}\frac{2\zeta_{\epsilon_4}(N_2)}{1-\xi}\right)^j \left(\frac{\norm{A}}{\gamma'} C_\gamma + C\right)\right\}, \quad \forall t_i \leq t \leq t_{i+1}.
    \end{equation*}
\end{proposition}
\begin{proof}
    Similar to the steps in proof of Lemma~\ref{lemm:ST_equiv_rev}, we obtain that 
    \begin{equation*}
        \frac{\norm{P_2^* x_{t_i}}}{\norm{x_{t_i}}} = \frac{\norm{\Pi_2(\Pi_u + \Pi_s) x_{t_i}}}{\norm{x_{t_i}}} = \frac{\norm{\Pi_2 \Pi_s x_{t_i}}}{\norm{x_{t_i}}} \leq \frac{\norm{\Pi_2 x_{t_i}}}{\norm{x_{t_i}}} = \frac{\norm{(I - \hat\Pi_1+ \hat\Pi_1 - \Pi_1) x_{t_i}}}{\norm{x_{t_i}}} \leq \gamma + \epsilon,
    \end{equation*}
    which shows the first part of the result. 
    
    We now focus on the second part (bounding $\Vert x_t\Vert$). 
    For the base case, We either have $t_1 = T$, thus $x_{t_1} = x_T$, in which case the stopping time criteria is already met after Stage 1 of algorithm~\ref{alg:LTS0}, or, if $t_1 > T$, there are two scenarios depending which of the two stopping criteria is violated at time $t_1-1$ . If $\frac{\norm{(I - \hat{\Pi}_1) x_{t_1 - 1}}}{\norm{x_{t_1 - 1}}} > \gamma$, by Proposition~\ref{prop:base_portion}, we have $\norm{x_{t_1 - 1}} < C_{\gamma}$, where $C_{\gamma}$ is defined in \eqref{eqn:C_gamma}, in which case, we have
    \begin{equation}
    \label{eqn:upper_base}
        \norm{x_{t_1}} = \norm{A x_{t_1 - 1} + \eta_{t_1 - 1}} \leq \norm{A} C_{\gamma} + C.
    \end{equation}
    In the second case, $\norm{x_{t_1-1}} \leq \frac{C}{\delta}$, so we have
    \begin{equation*}
    \label{eqn:upper_base2}
        \norm{x_{t_1}} < \norm{A}\frac{C}{\delta} + C.
    \end{equation*}
    Therefore, to sum up the base case, we have
    \begin{equation*}
        \norm{x_{t_1}} \leq \max\left\{\norm{A}\frac{C}{\delta} + C, \norm{A} C_{\gamma} + C, \norm{x_T}\right\}
    \end{equation*}
    
    For the induction case, given $\norm{x_{t_i}}$, there are again two cases depending on which criterion is violated at time $t_{i+1}-1$. If $\norm{x_{t_{i+1}-1}} \leq \frac{C}{\delta}$, we have
    \begin{equation*}
        \norm{x_{t_{i+1}}} < \norm{A}\frac{C}{\delta} + C.
    \end{equation*}
    Otherwise, if $\frac{\norm{(I - \hat{\Pi}_1) x_{t_{i+1}-1}}}{\norm{x_{t_{i+1}-1}}} > \gamma$, by Proposition~\ref{prop:ST_induction}, we obtain that 
    \begin{equation}
        \norm{x_{t_{i+1}-1}} < \frac{1}{\gamma'} \left(\frac{2\zeta_{\epsilon_4}(N_2)}{1-\xi}\norm{x_i} + C_{\gamma}\right),
    \end{equation}
    where $\gamma'$ is defined in Lemma~\ref{lemm:ST_equiv_rev}.
    
    By the definition of $\omega_i$, the maximum of the above inequalities also holds for all $x_{t}$ such that $t < t_{i+1}$. Therefore, 
    \begin{equation*}
        \norm{x_{t_{i+1}}} < \max\left\{\norm{A}\frac{C}{\delta} + C, \frac{\norm{A}}{\gamma'} \left(\frac{2\zeta_{\epsilon_4}(N_2)}{1-\xi}\norm{x_i} + C_{\gamma}\right) + C\right\},
    \end{equation*}
    as required. Note that the same bound above also holds for all $t_i<t<t_{i+1}$. Hence we get the desired result after a simple recursive expansion. 
\end{proof}

We are now ready to bound $\norm{\hat{B}_\tau - B_{\tau}}$.
\begin{proposition}
    \label{prop:G6}
    Under the premise of Theorem~\ref{thm:main}, 
    \begin{equation*}
        \norm{\hat{B}_\tau - B_\tau} < C_B \left(|\lambda_1|+\epsilon_1\right)^{\tau-1} \delta,
    \end{equation*}
    where $C_B := (\zeta_{\epsilon_1}^2 (A) (3\tau\norm{A} + \norm{B} + \tau C + 1)+ (\tau+1)C_{\Delta})\frac{\sqrt{m}}{\alpha}$.
\end{proposition}
\begin{proof}
    We have
    \begin{align*}
        \norm{b_i - \hat{b}_i}
        = & \frac{1}{\alpha\norm{x_{t_i}}}\Bigg \lVert P_1^* x_{t_i + \tau} - M_1^\tau P_1^* x_{t_i} - \Delta_{\tau} P_2^* x_{t_i} - \sum_{j = 1}^{\tau-1} (M_1^{\tau-j} P_1^* \eta_{t_i + j} - \Delta_{\tau-j} P_2^* \eta_{t_i + j})
        \\
        & - \left(\hat{P}_1^* x_{t_i+\tau} - \hat{M}_1^\tau \hat{P}_1^* x_{t_i}\right) \Bigg \rVert
        \\
        \leq & \frac{1}{\alpha\norm{x_{t_i}}} \Bigg(\norm{(P_1 - \hat{P}_1)^* \left(A^\tau x_{t_i} + B_\tau u_{t_i}\right)} + \norm{\sum_{j=1}^{\tau-1} M_1^{\tau-j} (P_1 -\hat{P}_1)^* \eta_{t_i + j}} + \norm{M_1^\tau P_1^* x_{t_i} - \hat{M}_1^\tau \hat{P}_1^* x_{t_i}}
        \\
        & + \norm{\Delta_{\tau} P_2^* x_{t_i}} + \sum_{j = 1}^{\tau-1} \norm{M_1^{\tau-j} P_1^* \eta_{t_i + j}} + \sum_{j = 1}^{\tau-1}\norm{\Delta_{\tau-j} P_2^* \eta_{t_i + j}}\Bigg).
    \end{align*}
    Here, the first term is bounded by
    \begin{align*}
        \norm{(P_1 - \hat{P}_1)^* \left(A^\tau x_{t_i} + B_\tau u_{t_i}\right)} 
        \leq & \norm{P_1 - \hat{P}_1} \left(\norm{A^\tau} + \norm{A^{\tau-1}B}\right) \norm{x_{t_i} }
        \\
        \leq & \norm{x_{t_i} } \zeta_{\epsilon_1} (A) \left(|\lambda_1| + \epsilon_1\right)^{\tau - 1} (\norm{A} + \norm{B})\delta,
    \end{align*}
    where in the last inequality we applied Corollary~\ref{coro:52.1} and Gelfand's formula; the second term is bounded by 
    \begin{align*}
        \norm{\sum_{j=1}^{\tau-1} M_1^{\tau-j} (P_1 -\hat{P}_1)^* \eta_{t_i + j}} 
        \leq & \sum_{j=1}^{\tau-1} \zeta_{\epsilon_1} (A) \left(|\lambda_1| + \epsilon_1\right)^{\tau - j} C \delta
        \\
        < & \tau \zeta_{\epsilon_1} (A) \left(|\lambda_1| + \epsilon_1\right)^{\tau - 1} C \delta,
    \end{align*}
    where we used Corollary~\ref{coro:52.1} and Gelfand's formula. 
    
    The third term is bounded above by 
    \begin{align*}
        \norm{M_1^\tau P_1^* x_{t_i} - \hat{M}_1^\tau \hat{P}_1^* x_{t_i}}
        \leq & \left(\norm{M_1^\tau (P_1 - \hat{P}_1)^* } + \norm{(M_1^\tau - \hat{M}_1^\tau) \hat{P}_1^*)}\right)\norm{x_{t_i}}
        \\
        < & \big(\zeta_{\epsilon_1} (A) \left(|\lambda_1| + \epsilon_1\right)^{\tau - 1} \norm{A} \delta + 3\tau \norm{A}\zeta_{\epsilon_1} (A) \left(|\lambda_1| + \epsilon_1\right)^{\tau - 1} \delta\big)\norm{x_{t_i}}
        \\
        \leq & \norm{x_{t_i}}\zeta_{\epsilon_1} (A)^2 \left(|\lambda_1| + \epsilon_1\right)^{\tau - 1} (3\tau+1) \norm{A} \delta,
    \end{align*}
    where we applied Gelfand's formula and \Cref{prop:G2}. 
    The fourth term is bounded by
    \begin{align}
        \frac{\norm{\Delta_\tau} \norm{P_2^* x_{t_i}}}{\norm{x_{t_i}}}
        \leq & C_{\Delta}(|\lambda_1| + \epsilon_1)^\tau (\gamma + \epsilon)
        \label{eqn:26}
        \\
        \leq & C_{\Delta}(|\lambda_1| + \epsilon_1)^\tau \delta,
        \label{eqn:28}
    \end{align}
    where in \eqref{eqn:26}, we used Proposition G.1 of \citet{LTI} and Proposition~\ref{prop:ST_final}, while and \eqref{eqn:28} we need to pick stopping time $\omega$ defined by $\gamma$:
    \begin{equation}
        \gamma \leq \delta - \epsilon = (\sqrt{2k}-1)\epsilon.
    \end{equation}
    For the second to last and the last term,
    \begin{align}
        \frac{1}{\norm{x_{t_i}}}\sum_{j = 1}^{\tau-1} \norm{M_1^{\tau-j} P_1^* \eta_{t_i + j}} \leq & \frac{1}{\norm{x_{t_i}}}\sum_{j=1}^{\tau-1} \zeta_{\epsilon_1} (A) \left(|\lambda_1| + \epsilon_1\right)^{\tau - j} C \notag
        \\
        < & \frac{1}{\norm{x_{t_i}}}\tau \zeta_{\epsilon_1} (A) \left(|\lambda_1| + \epsilon_1\right)^{\tau - 1} C \notag
        \\
        < & \tau \zeta_{\epsilon_1} (A) \left(|\lambda_1| + \epsilon_1\right)^{\tau - 1} \delta,
        \label{eqn:ST_ms}
    \end{align}
    \begin{align}
        \frac{1}{\norm{x_{t_i}}}\sum_{j = 1}^{\tau-1}\norm{\Delta_{\tau-j} P_2^* \eta_{t_i + j}}
        \leq & \frac{1}{\norm{x_{t_i}}} \tau C_{\Delta}(|\lambda_1| + \epsilon_1)^\tau C
        \notag
        \\
        \leq & \tau C_{\Delta}(|\lambda_1| + \epsilon_1)^\tau \delta,
        \label{eqn:ST_xt}
    \end{align}
    where in \eqref{eqn:ST_ms} and \eqref{eqn:ST_xt}, we need
    \begin{equation}
    \label{eqn:additional_ST}
        \frac{C}{\norm{x_{t_i}}} < \delta.
    \end{equation}
    We notice that \eqref{eqn:additional_ST} happens with high probability since the system runs mostly in open loop. If the above inequality is not satisfied, we can keep the system running in open loop until it is. If the above is never satisfied, then the system is stable. More formally, as the first stopping time $t_1$ stated in \Cref{prop:ST_final} is never reached, the bound for $\norm{x_{t_1}}$ holds for all $x_t$. 

    Finally, to bound the error of the whole matrix, we simply apply the definition
    \begin{align*}
        \norm{\hat{B}_{\tau} - B_{\tau}} =& \max_{\norm{u}=1}\norm{(\hat{B}_{\tau} - B_{\tau})u} \leq \max_{\norm{u}=1}\sum_{i=1}^m |u_i| \norm{\hat{b}_i - b_i}
        \\
        <& (\zeta_{\epsilon_1}^2 (A) (3\tau\norm{A} + \norm{B} + \tau C + 1)+ (\tau+1)C_{\Delta})\left(|\lambda_1| + \epsilon_1\right)^{\tau - 1} \delta \frac{\sqrt{m}}{\alpha}.
    \end{align*}
\end{proof}

\section{Proof of Main Theorem}
\label{Appendix:Main}
We assumed the system $(A,B)$ is controllable. As we are stabilizing the system in $(M^\tau, B_\tau)$, we need to first show that $(M^\tau, B_\tau)$ is stabilizable.  

\begin{proposition}
\label{prop:controllable_Mtau}
    If $(A,B)$ is controllable, then $(\hat{M}_1^\tau, R_1\hat{B}_\tau)$ is stabilizable.
\end{proposition}

\begin{proof}
    Since $(A,B)$ is controllable, by the PBH test criteria, there exists $b$, such that for all unit left eigenvector $\bar{w}$ of $A$, $\norm{\bar{w}^* B} > b$. 

    Let $w^*$ denote an arbitrary unit left eigenvector of $N_1$ with eigenvalue $\lambda$, so 
    \small
    \begin{equation*}
        w^* N_1 = \lambda w \quad \Rightarrow \quad (R_1^* w)^* A = w^* R_1 Q_1 N_1 R_1 = \lambda (R_1^* w)^* .
    \end{equation*}
    \normalsize
    Therefore, $R_1^* w$ is a left eigenvector of $A$, which leads to 
    \begin{equation*}
        \norm{w^* R_1 B} = \norm{(R_1^* w)^* B} > \norm{R_1^* w} b .
    \end{equation*}
    By the construction of $R_1$, as $R$ is invertible, we see that all singular values of $R_1$ are nonzero. Therefore, $\norm{R_1^* w} b > 0$. Correspondingly, $(N_1, R_1 B)$ is controllable.

    We then consider the system under $\tau$-hop control. 
    Since $w$ is the left eigenvector of $N_1$, it is also the left eigenvector of $N_1^\tau$. In particular, $w^* N^{\tau-1}$ is a left eigenvector of $N_1$. Since $N_1$ is the expanding portion of $A$, we derive the following lower bound:
    \small
    \begin{equation*}
        \norm{w^* \left(N_1^{\tau-1} R_1 B\right)} =  \norm{\left(w^* N_1^{\tau-1}\right) R_1 B } \geq \lambda_k^{\tau-1} \norm{R_1^* w} b .
    \end{equation*}
    \normalsize
    Recall that $B_\tau = P_1^* A^{\tau-1}B$. 
    \begin{align*}
        B_\tau &= P_1^* 
        \begin{bmatrix}
            Q_1 & Q_2
        \end{bmatrix}
        \begin{bmatrix}
            N_1^{\tau-1} & \\ & N_2^{\tau-1}
        \end{bmatrix}
        \begin{bmatrix}
            R_1 B \\ R_2 B
        \end{bmatrix}
        \\
        &= \begin{bmatrix}
            P_1^* Q_1 & P_1^* Q_2
        \end{bmatrix}
        \begin{bmatrix}
            N_1^{\tau-1} R_1 B \\ N_2^{\tau-1} R_2 B
        \end{bmatrix}
        \\
        &= N_1^{\tau-1} R_1 B + P_1^* Q_2 N_2^{\tau-1} R_2 B.
    \end{align*}
    By Gelfand's Formula, $\norm{N_2^{\tau-1}} \leq \zeta_{\epsilon_4}(N_2) \left(\lambda_{k+1} + \epsilon_4\right)^{\tau-1}$. Moreover, since $E_u^{\perp}$ and $E_s$ are $\xi$-close, by Lemma A.1 of \cite{LTI}, $P_1^* Q_2 \leq \sqrt{2 \xi}$.

    Therefore, we know that 
    \begin{align*}
        \norm{w^* B} =& \norm{w^*\left(N_1^{\tau-1} R_1 B + P_1^* Q_2 N_2^{\tau-1} R_2 B\right)} 
        \\
         \geq& |\lambda_k|^{\tau-1} \norm{R_1^* w} b 
         \\
         & - \sqrt{2 \xi} \norm{Q_2}\norm{R_2} \norm{B} \zeta_{\epsilon_4}(N_2) \left(|\lambda_{k+1}| + \epsilon_4\right)^{\tau-1} 
         \\
         >& \frac{1}{2}\norm{R_1^* w} b ,
    \end{align*}
    where the last inequality requires $\epsilon_4 < 1 - \lambda_{k+1}$, and
    \begin{equation}
        \label{eqn:tau_eps_4}
        \tau \geq \frac{\log \frac{\norm{R_1^* w} b}{2\sqrt{2 \xi} \norm{Q_2}\norm{R_2} \norm{B} \zeta_{\epsilon_4}(N_2)}}{\log \frac{|\lambda_k|}{|\lambda_{k+1}| + \epsilon_4}} .
    \end{equation}
    Therefore, we conclude $(M_1^\tau, R_1 B_\tau)$ is also controllable, as $M_1 = N_1$. 
    
    Lastly, we prove $\left(\hat{M}_1^\tau, R_1\hat{B}_\tau\right)$ is stabilizable. Denote $\mathcal{A} := M_1^\tau - R_1B_\tau K_1$.
    Since $(M_1^\tau, R_1 B_\tau)$ is controllable, we know there exists $K_1$ such that $\rho\left(\mathcal{A}\right) < 1$. Since an asymptotically stable linear system is also exponentially stable, by the Lyapunov equation, for every $k \times k$ matrix $G > 0$, the following discrete Lyapunov equation has a unique solution $H = H^* > 0$. 
    \begin{equation*}
        \mathcal{A}^* H \mathcal{A} + G - H = 0
    \end{equation*}
    In particular, we pick $G$ such that $\sigma_{\min}(G) > 2$ 
    and $W(v) := \frac{1}{\min\{1, \sigma_{\min}(H)\}} v^* H v$ is a Lyapunov function of $\mathcal{A}$. Moreover, $W(v)$ satisfies the following criteria regarding $\norm{v}$ and forward difference with respect to $\mathcal{A}$:
    \begin{equation*}
        \norm{v}^2 \leq W(v) \leq \kappa(H) \norm{v}^2 ,
    \end{equation*}
    \vspace{0.05in}
    \[
    \displaystyle
    \begin{array}{rcl}
            W\left(\mathcal{A}v\right) - W(v) & = & \frac{v^*\mathcal{A}^* H\mathcal{A} v - v^* H v}{\min\{1, \sigma_{\min}(H)\}}\\
             & \leq & -v^* G v \\
             &<& -2\norm{v}^2 ,
    \end{array}
    \]
    where $\kappa(H)$ is the condition number of $H$.
    
    We now consider the forward difference with respect to $ \hat{\mathcal{A}} = \hat{M}_1^\tau - R_1 \hat{B}_\tau K_1$, as a consequence of Jensen's inequality, for any $\iota > 0$, 
    \begin{equation*}
       \begin{split}
            W\left(\hat{\mathcal{A}}v\right)=& W\left(\mathcal{A} v + \left(\hat{\mathcal{A}} - \mathcal{A}\right)v\right)
            \\
            \leq& (1 + \iota^2) W(\mathcal{A}v) + \left(1 + \frac{1}{\iota^2}\right) W\left(\left(\hat{\mathcal{A}} - \mathcal{A}\right)v\right) ,
       \end{split}
    \end{equation*}
    and
    \begin{align*}
        & W\left(\hat{\mathcal{A}}v\right) - W(v)\\
        =& W(\mathcal{A}v) - W(v) + W\left(\hat{\mathcal{A}}v\right) - W(\mathcal{A}v) \\
        \leq&W(\mathcal{A}v) - W(v) + \iota^2 W(\mathcal{A}v) + \left(1 + \frac{1}{\iota^2}\right) W\left(\left(\hat{\mathcal{A}} - \mathcal{A}\right)v\right)\\
        <& -2 \norm{v}^2 + \iota^2 \kappa(H) \norm{v}^2 + \left(1 + \frac{1}{\iota^2}\right) \norm{\hat{\mathcal{A}} - \mathcal{A}}^2 \norm{v}^2\\
        \leq& -\norm{v}^2 ,
    \end{align*}
    The last inequality requires
    \begin{equation*}
        \iota^2 < \frac{1}{2 \kappa(H)}, \qquad \norm{\hat{\mathcal{A}} - \mathcal{A}}^2 < \frac{1}{2}\frac{\iota^2}{1 + \iota^2} .
    \end{equation*}
    By Proposition~\ref{prop:G2} and \ref{prop:G6}, we get
    \[
    \displaystyle
    \begin{array}{rcl}
         \norm{\hat{M}_1^\tau - M_1^\tau} &<& 3 \tau \norm{A} \zeta_{\epsilon_1}(A)^2 \left(|\lambda_1| + \epsilon_1\right)^{\tau-1} \delta  ,\\
         \norm{\hat{B}_\tau - B_\tau} &<& C_B(|\lambda_1| + \epsilon_1)^{\tau-1} \delta .
    \end{array}
    \]
    So we require
    \begin{equation}
        \label{eqn:delta_stabilizable}
        \delta < \frac{\frac{1}{6}\frac{\iota^2}{1 + \iota^2}}{\tau \norm{A} \zeta_{\epsilon_1}(A)^2 \left(|\lambda_1| + \epsilon_1\right)^{\tau-1} + \norm{K_1} C_B(|\lambda_1| + \epsilon_1)^{\tau-1}} .
    \end{equation}
    When all requirements above are satisfied, by Theorem 2 of \cite{converse_lyapunov}, we conclude $\left(\hat{M}_1^\tau, R_1\hat{B}_\tau\right)$ is stabilizable.
\end{proof}

As the control matrix $\hat{K}_1$ is obtained by the learner, we denote constant $\mathcal{K}$ such that $\norm{\hat{K}_1} < \mathcal{K}$ to be a user-defined constant. 

After the proof of the stabilizability of the system after transformation, we are now ready to prove the main theorem. 

\begin{proof}[proof of Theorem \ref{thm:main}]
We shall bound each of the four terms in $\hat{L}$ defined in \eqref{eqn:L_hat}. We first guarantee that the diagonal blocks are stable. For the top-left block, by \Cref{prop:controllable_Mtau}, there exists positive-definite matrix $\Bar{U}$ such that $\norm{\hat{M}_1^{\tau} - \hat{B}_{\tau} \hat{K}_1}_{\Bar{U}} = \mathcal{U} < 1$, where $\norm{\cdot}_{\Bar{U}}$ denotes the weighted norm induced by $\Bar{U}$. Therefore, 

\begin{align}
    \rho(\hat{L}_{1,1}) \leq& \norm{M_1^\tau + P_1^* A^{\tau-1} B \hat{K}_1 \hat{P}_1^* P_1}_{\Bar{U}}
    \\
    \leq& \norm{M_1^\tau - \hat{M}_1^\tau}_{\Bar{U}} + \norm{\hat{M}_1^{\tau} - \hat{B}_{\tau} \hat{K}_1}_{\Bar{U}} + \norm{(B_\tau - \hat{B}_\tau) \hat{K}_1}_{\Bar{U}} + \norm{B_\tau \hat{K}_1 (I - \hat{P}_1^* P_1)}_{\Bar{U}} \notag \\
    \leq & \kappa(\Bar{U})^{\frac{1}{2}}\left(\norm{M_1^\tau - \hat{M}_1^\tau} + \norm{B_\tau - \hat{B}_\tau} \norm{\hat{K}_1} + \norm{B_\tau}\norm{\hat{K}_1}\norm{I - \hat{P}_1^* P_1}\right) + \mathcal{U}\notag 
    \\
    \label{eqn:32}
    \leq&  3 \kappa(\Bar{U})^{\frac{1}{2}}\tau \norm{A} \zeta_{\epsilon_1}(A)^2 (|\lambda_1| + \epsilon_1)^{\tau-1} \delta + \kappa(\Bar{U})^{\frac{1}{2}} C_B \mathcal{K}(|\lambda_1| + \epsilon_1)^{\tau-1} \delta \notag\\
        & + \kappa(\Bar{U})^{\frac{1}{2}}\zeta_{\epsilon_1}(A) (|\lambda_1| + \epsilon_1)^{\tau-1} \norm{B} \mathcal{K} \delta + \mathcal{U}
        \\
        \label{eqn:33}
        <&  \kappa(\Bar{U})^{\frac{1}{2}}(C_B \mathcal{K} + \zeta_{\epsilon_1}(A)\norm{B}\mathcal{K} + 1) (\lambda_1| + \epsilon)^{\tau-1}\delta + \mathcal{U}
        \\
        \label{eqn:34}
        <&  \frac{1}{2} + \frac{\mathcal{U}}{2},
\end{align}

where in \eqref{eqn:32} we apply proposition E.1 of \cite{LTI} and \Cref{prop:G2} and \Cref{prop:G6}; In \eqref{eqn:33}, we require
\begin{equation}
\label{eqn:35}
    \frac{1}{\tau}(|\lambda_1| + \epsilon_1)^{\tau-1} > 3 \norm{A}\zeta_{\epsilon_1}(A)^2 .
\end{equation}
In \eqref{eqn:34}, we require
\begin{equation}
    \label{eqn:delta11}
    \delta < \frac{(1-\mathcal{U})(\lambda_1| + \epsilon)^{-(\tau-1)}}{2\kappa(\Bar{U})^{\frac{1}{2}} (C_B \mathcal{K} + \zeta_{\epsilon_1}(A)\norm{B}\mathcal{K} + 1)}.
\end{equation}
For the bottom-right block, it is straightforward to see that
\begin{align*}
    \rho(\hat{L}_{2,2}) \leq& \norm{M_2^\tau} + \norm{P_2^* A^{\tau-1}} \norm{B} \norm{\hat{K}_1} \norm{\hat{P}_1^* P_2} \\
    \leq & \zeta_{\epsilon_2}(M_2) (|\lambda_{k+1}| + \epsilon_2)^\tau + \zeta_{\epsilon_2}(M_2) \norm{B} \mathcal{K} (|\lambda_{k+1}| + \epsilon_2)^{\tau-1} \delta \\
    < & \frac{1}{2} ,
\end{align*}
where the last inequality requires
\begin{equation}
\label{eqn:37}
    \tau > \frac{\log1/(4 \zeta_{\epsilon_2}(M_2))}{\log(|\lambda_{k+1}| + \epsilon_2)} ,
\end{equation}
\begin{equation}
\label{eqn:38}
    \delta < \frac{1}{4 \zeta_{\epsilon_2}(M_2)\norm{B}\mathcal{K}} (|\lambda_{k+1}| + \epsilon_2)^{-(\tau-1)}.
\end{equation}
Now it suffices to bound the spectral norms of off-diagonal blocks. Note that, by applying Proposition G.1 of \citet{LTI}, the top right block is bounded as
\begin{align*}
    \rho(\hat{L}_{2,1}) \leq & \norm{\Delta_\tau} + \norm{B_\tau} \norm{\hat{K}_1} \norm{\hat{P}_1^* P_2}\\
     < & C_{\Delta}(|\lambda_1| + \epsilon_1)^\tau + \zeta_{\epsilon_1}(A) \norm{B} \mathcal{K} (|\lambda_1| + \epsilon_1)^{\tau-1}\delta \\
     < & (C_\Delta + 1) (|\lambda_1| + \epsilon_1)^\tau ,
\end{align*}
where the last inequality requires
\begin{equation}
\label{eqn:delta21}
    \delta < \frac{1}{\zeta_{\epsilon_1}(A) \norm{B} \mathcal{K}} (|\lambda_1| + \epsilon_1)^{-(\tau-1)}.
\end{equation}
The bottom-left block is bounded as
\begin{align*}
    \rho(\hat{L}_{1,2}) \leq& \norm{P_2^* A^{\tau-1}}\norm{B}\norm{\hat{K}_1}\\
    <& \zeta_{\epsilon_2}(M_2) \norm{B} \mathcal{K} (|\lambda_{k+1}|+\epsilon_2)^{\tau-1}.
\end{align*}
By Lemma 5.3 of \citet{LTI}, we can guarantee that
\begin{equation}
    \rho\left(\hat{L}_\tau\right) \leq \frac{1}{2}+\frac{\mathcal{U}}{2} + \chi\left(\hat{L}_{\tau}\right) \frac{(C_\Delta + 1)\zeta_{\epsilon_2}(M_2)\norm{B}\mathcal{K}}{|\lambda_1| + \epsilon_1} \left((|\lambda_1| + \epsilon_1)(|\lambda_{k+1}| + \epsilon_2)\right)^{\tau-1} < 1 ,
\end{equation}
which requires
\begin{equation}
\label{eqn:40}
    \tau > \frac{\log \frac{(1-\mathcal{U})(|\lambda_1| + \epsilon_1)(|\lambda_{k+1}| + \epsilon_2)}{2\chi(\hat{L}_\tau) (C_\Delta + 1) \zeta_{\epsilon_2}(M_2)\norm{B} \mathcal{K}}}{\log ((|\lambda_1| + \epsilon_1)(|\lambda_{k+1}| + \epsilon_2))} .
\end{equation}

Note that the above constraints make sense only if $|\lambda_1||\lambda_{k+1}| < 1$. 
Therefore, when all constraints above are satisfied, system \eqref{eqn:tau_hop_closed} is ultimately bounded, and so is system \eqref{eqn:system_dynamics}. 

We will then collect all the constraints. Combining \eqref{eqn:35} \eqref{eqn:37} and \eqref{eqn:40}, we obtain

\small
\begin{align*}
    \tau >& \max \Bigg\{\frac{\log1/(4 \zeta_{\epsilon_2}(M_2))}{\log(|\lambda_{k+1}| + \epsilon_2)}, \frac{\log \frac{(\mathcal{U}+1)(|\lambda_1| + \epsilon_1)(|\lambda_{k+1}| + \epsilon_2)}{2\chi(\hat{L}_\tau) (C_\Delta + 1) \zeta_{\epsilon_2}(M_2)\norm{B} \mathcal{K}}}{\log ((|\lambda_1| + \epsilon_1)(|\lambda_{k+1}| + \epsilon_2))},\\
    &   - \frac{1}{\log (|\lambda_1|+\epsilon_1)} W_{-1}\left(-\frac{\log (|\lambda_1|+\epsilon_1)}{3\norm{A}\zeta_{\epsilon_1}(A)^2 (|\lambda_1|+\epsilon_1)}\right)\Bigg\} ,
\end{align*}
\normalsize

where $W_{-1}$ denotes the non-principle branch of the Lambert-W function. Here we utilize the fact that, for $x > \frac{1}{\log a}, y = \frac{a^*}{x}$ is monotone increasing with inverse function $x = - \frac{1}{\log a}W_{-1}\left(-\frac{\log a}{y}\right)$, which can be upper bounded by Theorem 1 in \cite{Lambert} as
\small
\begin{equation}
\label{eqn:tau_final}
    \begin{split}
        &\tau >
        \frac{\log \frac{\sqrt{\xi}}{1 - \xi} + \log \frac{1}{c} + \log \chi \left(\hat{L}_\tau\right) + 5 \log \bar{\zeta} + \log \frac{\norm{A}}{|\lambda_1| - |\lambda_{k+1}|} + C_\tau}{\log |\lambda_1|}
        \\
        &= O(1) ,
    \end{split}
\end{equation}
\normalsize
where $\bar{\zeta} := \max\left\{\zeta_{\epsilon_1}(A), \zeta_{\epsilon_2}(M_2), \zeta_{\epsilon_2}(N_2), \zeta_{\epsilon_3}(N_1^{-1})\right\}$, and $C_\tau$ is a numerical constant. 

We then collect all the bound on $\gamma, \alpha,\delta$ as follows:
\begin{equation}
    \label{eqn:bdd_gamma}
    \gamma > \epsilon ,
\end{equation}
\begin{equation}
    \label{eqn:bdd_alpha}
    \alpha < \frac{1}{\norm{B}} 
    = O(1) .
\end{equation}
Combining \eqref{eqn:delta_stabilizable}, \eqref{eqn:delta11}, \eqref{eqn:38}, \eqref{eqn:delta21} yields the following bound on $\delta$:
\begin{equation*}
    \begin{split}
        \delta < \max& \Bigg\{\frac{\frac{1}{6}\frac{\iota^2}{1 + \iota^2}}{\tau \norm{A} \zeta_{\epsilon_1}(A)^2 \left(|\lambda_1| + \epsilon_1\right)^{\tau-1} + \norm{K_1} C_B(|\lambda_1| + \epsilon_1)^{\tau-1}}, \frac{(1-\mathcal{U})(\lambda_1| + \epsilon)^{-(\tau-1)}}{2\kappa(\Bar{U})^{\frac{1}{2}} (C_B \mathcal{K} + \zeta_{\epsilon_1}(A)\norm{B}\mathcal{K} + 1)}, 
        \\
        &\frac{1}{4 \zeta_{\epsilon_2}(M_2)\norm{B}\mathcal{K}} (|\lambda_{k+1}| + \epsilon_2)^{-(\tau-1)}, \frac{1}{\zeta_{\epsilon_1}(A) \norm{B} \mathcal{K}} (|\lambda_1| + \epsilon_1)^{-(\tau-1)} \Bigg\}.
    \end{split}
\end{equation*}
which can be simplified to
\begin{equation}
    \label{eqn:bdd_delta}
    \delta < \frac{C_\delta}{\sqrt{m}\bar{\zeta}^3 (\norm{A} + \norm{B})} |\lambda_1|^{-2\tau} = O(m^{-1/2} |\lambda_1|^{-2\tau}) ,
\end{equation}
where $C_{\delta}$ is a constant collecting minor factors. Recall that $\delta = \sqrt{2k}\epsilon$. Substitute the above in \eqref{eqn:T_complete} transfers the bound on $\delta$ into a bound on $T$:
\begin{equation}
\label{eqn:T_in_main}
    T > \frac{2\log \bigg(\frac{8k^{\frac{k}{2}+4} (n-k)\left(\frac{C}{1-|\lambda_{k+1}|}\right)\left(\frac{\sqrt{m}\bar{\zeta}^3 (\norm{A} + \norm{B})}{C_{\delta}|\lambda_1|^{-2\tau}}\right)}{\sqrt{\pi}\theta \gap\epsilon}\bigg)}{\log |\lambda_k|}
    =
    O\left(k \log k + \log (n-k) + \log m - \log \gap\right)
\end{equation}

Different from \citet{LTI}, we do not explicitly choose $\omega$ but let $(\omega_i)_{i\in \{1,\dots,m\}}$ be the stopping time defined in \Cref{prop:ST_final}.

Combining the above constant with Theorem \ref{thm:projection}, we conclude that Algorithm \ref{alg:LTS0} controls $x$ with the following bound:
\begin{equation*}
    \begin{split}
        \norm{x} \leq& \exp\left(O \left(T + \sum_{i=1}^m\omega_i + \tau m\right)\right) \\
        \leq& \exp \Bigg(O\Bigg(\frac{1}{\log|\lambda_k|}\Bigg( - \log \gap + k \log k - \log \theta + \log(n-k)
        \\
        & + \log |\lambda_1| + \log C - \log \left(1-|\lambda_{k+1}|\right) + (1 + \log |\lambda_1|)m\Bigg)\Bigg)\Bigg.
    \end{split}
\end{equation*}

Assuming that the eigenvalue-related terms are constants, the algorithm achieves $\exp(O(k \log k + \log(n-k) + m - \log \gap))$ space complexity for $\norm{x}$. 

This finishes the proof of \Cref{thm:main}.
\end{proof}

\section{Additional Mathematical Background}
\label{Appendix:DKT}
In this section, we introduce some relevant math background used in this paper. The notation of this section is independent of the rest of the paper. 

\begin{theorem}[Davis-Kahan]
    Let $A$ be an $n \times n$ Hermitian matrix, and suppose we have the following spectral decomposition for $A$
    \begin{equation*}
        A = \sum_{i=1}^n \lambda_i u_i u_i^*,
    \end{equation*}
    where $\lambda_i$'s are the eigenvalues of $A$ such that $\lambda_1 > \dots > \lambda_n$, and $u_i$'s are corresponding eigenvectors. Let $H$ be another $n \times n$ perturbation matrix, and the spectral decomposition of $A + H$ is
    \begin{equation*}
        A + H = \sum_{i=1}^n \mu_i v_i v_i^*.
    \end{equation*}
    Define 
    \begin{equation*}
        P = \sum_{i = 1}^k u_i u_i^* := U U^*
    \end{equation*}
    to be the orthogonal projection operator to the $k$-dimensional eigenspace spanned by $u_1 \dots, u_k$. Similarly, define $Q = \sum_{i=1}^k v_i v_i^* := V V^*$. 

    Suppose there exists $\delta > 0$, such that $|\lambda_i - \mu_j| > \delta$ for all $i \in \{1,\dots,k\}, j \in \{k+1,\dots,n\}$, then the operator norm of $\norm{P-Q}_{op}$ satisfy
    \begin{equation*}
        \norm{P-Q}_{op} \leq \norm{P-Q}_F \leq \frac{\sqrt{2k}\norm{H}_{op}}{\delta},
    \end{equation*}
    where $\norm{\cdot}_F$ denotes the Frobenius norm. 
\end{theorem}
This is a relatively common theorem, and the proof detail can be found at, for instance, \cite{Davis-Kahan}.

\begin{lemma}[Gelfand's formula]
\label{lemma:Gelfand}
    For any square matrix $X$, we have
    \begin{equation*}
        \rho(X) = \lim_{t \rightarrow \infty} \norm{X^t}^{1/t}.
    \end{equation*}
    In other words, for any $\epsilon > 0$, there exists a constant $\zeta_{\epsilon}(X)$ such that
    \begin{equation*}
        \sigma_{\max}(X^t) = \norm{X} \leq \zeta_{\epsilon}(X)(\rho(X) + \epsilon)^t. 
    \end{equation*}
    Further, if $X$ is invertible, let $\lambda_{\min}(X)$ denote the eigenvalue of $X$ with minimum modulus, then 
    \begin{equation*}
        \sigma_{\min}(X^t) \geq \frac{1}{\zeta_{\epsilon}(X^{-1})} \left(\frac{|\lambda_{\min}(X)|}{1 + \epsilon |\lambda_{\min}(X)|}\right)^t. 
    \end{equation*}
\end{lemma}
The proof can be found in existing literatures (e.g. \cite{Roger12}. 
\section{Indexing}
\label{Appendix:index}
For the convenience of readers, we provide a table summarizing all constants appearing in the bounds. 
\begin{table}[!htbp]
\caption{Lists of parameters and constants appearing in the bound.}
\label{table:algorithm}
\vskip 0.15in
\begin{center}
\begin{small}
\begin{tabular}{lcccr}
\toprule
\textbf{Constant} & \textbf{Appearance} & \textbf{Explanation} \\
\midrule
$T$    & Stage 1 & $T$ initialization steps to separate unstable components. \\
$\omega_i$ & Stage 3 & Stopping time in each iteration to learn $B_\tau$.\\
$\alpha$    & Stage 3 & $u_{t_i} = \alpha \norm{x_{t_i}}e_i$ to estimate columns of $B_\tau$.\\
$\tau$    & Stage 3 & $\tau$-steps between consecutive control inputs are injected.         \\
\bottomrule
\end{tabular}
\end{small}
\end{center}
\vskip -0.1 in
\caption{System parameters.}
\label{table:parameters}
\vskip 0.15in
\begin{center}
\begin{small}
\begin{tabular}{lcccr}
\toprule
\textbf{Constant} & \textbf{Appearance} & \textbf{Explanation} \\
\midrule
$C$    & \Cref{sec:problem_formulation} & Upper bound the magnitude of noise.\\
$\lambda_i$    & \Cref{subsec:decompose} & (Complex) eigenvalue of $A$ with $i$-th largest modulus.         \\
$\xi$ & Definition 3.1 of \citet{LTI} & $E_u^\perp$ and $E_s$ are $\xi$-close subspaces, i.e. $\sigma_{\min} P_2^* Q_1 > 1-\xi$. \\
$\zeta_{\epsilon}(\cdot)$ & \Cref{lemma:Gelfand} & Gelfand constant for the norm of matrix exponents \\
\bottomrule
\end{tabular}
\end{small}
\end{center}

\vskip -0.1 in
\caption{Shorthand notations (introduced in proofs).}
\label{table:shorthand}
\vskip 0.15in
\begin{center}
\begin{small}
\begin{tabular}{lcccr}
\toprule
\textbf{Constant} & \textbf{Appearance} & \textbf{Explanation} \\
\midrule
$C_{\Delta}$    & Proposition G.1 of \citet{LTI} & $C_{\Delta}:= \zeta_{\epsilon_1}(M_1) \zeta_{\epsilon_2}(M_2) \frac{(2-\xi)\sqrt{2\xi}\norm{A}}{1-\xi}\frac{2|\lambda_{k+1}|}{|\lambda_1+\epsilon_1 - |\lambda_{k+1}|-\epsilon_2}$.\\
$C_{\gamma}$    & \eqref{eqn:C_gamma} in the proof of \Cref{prop:base_portion} & $C_{\gamma} :=\frac{\zeta_{\epsilon_4}(N_2) C}{\gamma'(1-\xi)} \frac{1}{1 - (|\lambda_{k+1}|+\epsilon_4)}$.\\
$C_B$    & \Cref{prop:G6} & $(\zeta_{\epsilon_1}^2 (A) (\norm{A} + \norm{B} + (C+2)\tau + 1)+ (\tau+1)C_{\Delta})\frac{\sqrt{m}}{\alpha}$.         \\
$\mathcal{K}$ & Stage 4 & Upper bounding $\norm{\hat{K}_1}$ chosen by the user. \\
$\mathcal{U}$ & Stage 4 & Upper bounding $\norm{\hat{M}_1^{\tau} - \hat{B}_{\tau} \hat{K}_1}_{\Bar{U}}$.
\\
$\gap$ & \Cref{thm:main} & $\gap := \left|\prod_{m_1 \neq m_2}(\lambda_{m_1}^{-1} - \lambda_{m_2}^{-1})\right|, m_1, m_2 \in \{1,\dots,k\}$. \\
\bottomrule
\end{tabular}
\end{small}
\end{center}
\end{table}

\end{document}